\documentclass[preprint,12pt]{elsarticle}

\usepackage{hyperref}
\usepackage[utf8]{inputenc}
\usepackage{graphicx,amsfonts,amssymb,amsthm,stmaryrd,times}
\usepackage{amsmath}
\usepackage{graphicx,psfrag}
\usepackage[T1]{fontenc}
\usepackage[all]{xy}
\usepackage{multirow}
\usepackage{rotating}
\usepackage{xcolor}
\usepackage{dsfont}
\usepackage{wasysym}
\usepackage{url}
\usepackage{type1cm}
\usepackage{mathrsfs}
\usepackage{amsmath}
\usepackage{upgreek}
\usepackage{lettrine}
\usepackage[ruled,vlined]{algorithm2e}
\theoremstyle{definition}
\newtheorem{theorem}{Theorem}
\newtheorem*{theorem*}{Theorem}

\def\amax{\kern 0em\hbox{\rm \kern .25em\lower.1ex\hbox{\rlap{$\vee$}}\kern -.07em\lower.2ex\hbox{$\square$}\kern.25em}}
\def\amin{\kern 0em\hbox{\rm \kern .25em\lower.1ex\hbox{\rlap{$\wedge$}}\kern -.07em\lower.2ex\hbox{$\square$}\kern.25em}}

\def\boxmax{\kern 0em\hbox{\rm \kern .25em\lower.1ex\hbox{\rlap{$\vee$}}\kern -.07em\lower.2ex\hbox{$\square$}\kern.25em}}
\def\boxmin{\kern 0em\hbox{\rm \kern .25em\lower.1ex\hbox{\rlap{$\wedge$}}\kern -.07em\lower.2ex\hbox{$\square$}\kern.25em}}
\def\dualimp{\kern 0em\hbox{\rm \kern .25em\lower.1ex\hbox{\rlap{$\Rightarrow$}}\kern 0em\lower-1.2ex\hbox{$\overline{\hspace{2ex}}$}\kern.25em}}

\def\circmax{\kern 0em\hbox{\rm \kern .25em\lower.1ex\hbox{\rlap{$\vee$}}\kern -.18em\lower-.1ex\hbox{$\bigcirc$}\kern.25em}}
\def\circmin{\kern 0em\hbox{\rm \kern .25em\lower.1ex\hbox{\rlap{$\wedge$}}\kern -.18em\lower-.0ex\hbox{$\bigcirc$}\kern.25em}}

\newcommand{\sgn}{\text{sgn}}

\newcommand{\ii}{\mathbf{i}}
\newcommand{\jj}{\mathbf{j}}
\newcommand{\kk}{\mathbf{k}}

\newcommand{\quat}[1]{{#1}_0 + {#1}_1 \ii + {#1}_2 \jj + {#1}_3 \kk}

\newcommand{\re}[1]{\text{Re}\left\{#1\right\}}
\newcommand{\ve}[1]{\text{Ve}\left\{#1\right\}}

\newcommand{\sprod}[2]{\left\langle #1, #2 \right\rangle}
\newcommand{\bb}{\begin{equation}}
\newcommand{\ee}{\end{equation}}

\newcommand{\R}{\mathbb{R}}

\newcommand{\vetu}{\boldsymbol{u}}
\newcommand{\vetx}{\boldsymbol{x}}
\newcommand{\vety}{\boldsymbol{y}}

\newcommand{\vetw}{\boldsymbol{w}}
\newcommand{\vetv}{\boldsymbol{v}}

\newcommand{\bpm}{\begin{bmatrix}}
\newcommand{\epm}{\end{bmatrix}}

\newcommand{\betav}{\mbox{\boldmath$\beta$}}

\newcommand{\vetz}{\boldsymbol{z}}

\newcommand{\imI}{\mathbf{I}}

\newcommand{\qeq}{\quad \mbox{and} \quad}

\journal{}

\begin{document}

\begin{frontmatter}


\title{Quaternion-Valued Recurrent Projection Neural Networks on Unit Quaternions}



\author{Marcos Eduardo Valle and Rodolfo Anibal Lobo}

\address{\textit{Institute of Mathematics, Statistics, and Scientific Computing} \\
\textit{University of Campinas}\\
Campinas, Brazil \\
valle@ime.unicamp.br, rodolfolobo@ug.uchile.cl
}

\begin{abstract}
Hypercomplex-valued neural networks, including quaternion-valued neural networks, can treat multi-dimensional data as a single entity. In this paper, we present the quaternion-valued recurrent projection neural networks (QRPNNs). Briefly, QRPNNs are obtained by combining the non-local projection learning with the quaternion-valued recurrent correlation neural network (QRCNNs). We show that QRPNNs overcome the cross-talk problem of QRCNNs. Thus, they are appropriate to implement associative memories. Furthermore, computational experiments reveal that QRPNNs exhibit greater storage capacity and noise tolerance than their corresponding QRCNNs. 
\end{abstract}

\begin{keyword}
Recurrent neural network \sep Hopfield network \sep associative memory \sep quaternion-valued neural network.

\end{keyword}

\end{frontmatter}

%\linenumbers

\section{Introduction} \label{sec:introduction}

The Hopfield neural network, developed in the early 1980s, is an important and widely-known recurrent neural network which can be used to implement associative memories \cite{hopfield82,hopfield85}. Successful applications of the Hopfield network include control \cite{gan17,song17}, computer vision and image processing \cite{wang15,jli16}, classification \cite{pajares10,zhang17}, and optimization \cite{hopfield85,serpen08,cli16}. 

Despite its many successful applications, the Hopfield network may suffer from a very low storage capacity when used to implement associative memories. Precisely, due to cross-talk between the stored items, the Hebbian learning adopted by Hopfield in his original work allows for the storage of approximately $n/(2\ln n)$ items, where $n$ denotes the length of the stored vectors \cite{mceliece87}. 

Several neural networks and learning rules have been proposed in the literature to increase the storage capacity of the original bipolar Hopfield network. For example, Personnaz et al. \cite{personnaz85} as well as Kanter and Sompolinsky \cite{kanter87} proposed the projection rule to determine the synaptic weights of the Hopfield networks. The projection rule increases the storage capacity of the Hopfield network to $n-1$ items. Another simple but effective improvement on the storage capacity of the original Hopfield networks was achieved by Chiueh and Goodman's recurrent correlation neural networks (RCNNs) \cite{chiueh91,chiueh93}. Briefly, an RCNN is obtained by decomposing the Hopfield network with Hebbian learning into a two layer recurrent neural network. The first layer computes the inner product (correlation) between the input and the memorized items followed by the evaluation of a non-decreasing continuous excitation function. The subsequent layer yields a weighted average of the stored items. Alternatively, certain RCNNs can be viewed as kernelized versions of the Hopfield network with Hebbian learning \cite{garcia04a,garcia04b,perfetti08}. 

It turns out that the associative memory models described in the previous paragraphs are all designed for the storage and recall of bipolar real-valued vectors. In many applications, however, we have to process multivalued or multidimensional data \cite{hirose12}. In view of this remark, the Hopfield neural network as well as the RCNNs have been extended to hypercomplex systems such as complex numbers and quaternions. 

Research on complex-valued Hopfield neural networks dates to the late 1980s \cite{noest88a,noest88b,aizenberg92}. In 1996, Jankowski et al. \cite{jankowski96} proposed a multistate complex-valued Hopfield network with Hebbian learning that corroborated to the development of many other hypercomplex-valued networks. Among the many papers on hypercomplex-valued versions of the Hopfield network, we would like to mention the following works which are strongly related to the models addressed in this paper. First, Lee developed the projection rule for (multistate) complex-valued Hopfield networks \cite{lee06}. Based on the works of Jankowski et al. and Lee, Isokawa et al. proposed a multistate quaternion-valued Hopfield neural network using either Hebbian learning or projection rule \cite{isokawa08b,isokawa13}. Unfortunately, the multistate quaternion-valued Hopfield neural network may fails to settle at an equilibrium state \cite{valle16wcci,valle18tnnls}. In contrast, the continuous-valued quaternionic model proposed independently by Valle and Kobayashi always comes to rest at a stable equilibrium point under the usual conditions on the synaptic weights \cite{valle14bracis,kobayashi16a}. Apart from hypercomplex-valued Hopfield networks, Valle proposed a complex-valued version of the RCNNs \cite{valle14nnB}. Recently, the RCNNs have been further extended to quaternions \cite{valle18wcci}. 

% At this point, we would like call the reader's attention to the following: Preliminary experiments revealed that hypercomplex-valued recurrent neural networks, interpreted as a dynamical systems, are less susceptible to chaotic behavior than their corresponding real-valued network \cite{castro20nn}. For example, hypercomplex-valued Hopfield networks usually require fewer updates to settle down at an equilibrium state than their corresponding real-valued neural networks. In view of this remark, we shall focus on quaternion-valued recurrent neural networks that can be used to implement associative memories. 

% Precisely, 
In this paper we propose an improved version of the quaternion-valued RCNNs \cite{valle18wcci}. Although real, complex, and quaternion-valued RCNNs can be used to implement high-capacity associative memories, they require a sufficiently large parameter which can be impossible in practical implementations \cite{chiueh93,valle18wcci}. In order to circumvent this problem, we combine the projection rule and the QRCNNs to obtain the new quaternion-valued recurrent projection neural networks (QRPNNs). As we will show, QRPNNs always have optimal absolute storage capacity. In other words, the fundamental memories are all stationary states (fixed points) of a QRPNN under mild conditions on the stored items and the excitation function. Furthermore, the noise tolerance of QRPNNs are usually higher than their corresponding QRCNNs. Also, bipolar RPNNs are strontly related to the kernel associative memories proposed by Garcia and Moreno and further investigated by Perfetti and Ricci \cite{garcia04a,garcia04b,perfetti08}.

We would like to point out that this paper corresponds to an extended version of the conference paper \cite{valle19bracis}. The most significant differences in this paper include: 
\begin{itemize}
\item A matrix-based formulation of QRCNNs and QRPNNs. 
\item A detailed algorithm describing the new QRPNN, which can also be used to implement QRCNNs. 
\item Formalized the results concerning the storage capacity of QRPNNs (Theorem \ref{thm:fixed_points}) and their relationship with QRCNNs and RKAMs in the bipolar case (Theorems \ref{thm:RPNNxRCNN} and \ref{thm:RKAMxRPNN}).
\item Additional computational experiments, including experiments concerning the storage and recall of color images from the CIFAR dataset \cite{cifar}.
\end{itemize}

The paper is organized as follows: Next section presents some basic concepts on quaternions. A brief review on the quaternion-valued Hopfield neural network (QHNN) and quaternion-valued recurrent correlation neural networks (QRCNNs) are given respectively in Sections \ref{sec:QHNN} and \ref{sec:QRCNNs}. Quaternion-valued recurrent projection neural networks (QRPNNs) are introduced in Section \ref{sec:QRPNNs}. Computational experiments are presented in Section \ref{sec:ComputationalExperiments}. The paper finishes with the concluding remarks in Section \ref{sec:concluding}.

%This document is a model and instructions for \LaTeX.
%Please observe the conference page limits. 

\section{Some Basic Concepts on Quaternions} \label{sec:quaternions}

Quaternions are hyper-complex numbers that extend the real and complex numbers systems.
A quaternion may be regarded as a 4-tuple of real numbers, i.e., $q=(q_0,q_1,q_2,q_3)$. Alternatively, a quaternion $q$ can be written as follows
\bb q = q_0 + q_1 \ii + q_2 \jj + q_3 \kk, \label{eq:quaterion} \ee
where $\ii, \jj$, and $\kk$ are imaginary numbers that satisfy the following identities:
\bb \ii^2 = \jj^2 = \kk^2 = \ii \jj \kk = -1. \ee
Note that $1,\ii,\jj$, and $\kk$ form a basis for the set of all quaternions, denoted by $\mathbb{H}$. 

A quaternion $q = \quat{q}$ can also be written as $q = q_0+\vec{q}$, where $q_0$ and $\vec{q}=q_1 \ii + q_2 \jj + q_3 \kk$ are called respectively the real part and the vector part of $q$. The real and the vector part of a quaternion $q$ are also denoted by $\re{q}:=q_0$ and $\ve{q}:=\vec{q}$.

The sum $p+q$ of two quaternions $p = \quat{p}$ and $q=\quat{q}$ is the quaternion obtained by adding their components, that is,
\bb \label{eq:sum} p+q = (p_0+q_0)+(p_1+q_1) \ii + (p_2+q_2) \jj + (p_3+q_3) \kk. \ee
Furthermore, the product $pq$ of two quaternions $p = p_0 + \vec{p}$ and $q = q_0 + \vec{q}$ is the quaternion given by 
\bb \label{eq:product} pq = p_0 q_0 - \vec{p} \cdot \vec{q} +p_0 \vec{q} + q_0 \vec{p} + \vec{p} \times \vec{q},\ee
where $\vec{p} \cdot \vec{q}$ and $\vec{p} \times \vec{q}$ denote respectively the scalar and cross products commonly defined in vector algebra. Quaternion algebra are implemented in many programming languages, including {\tt MATLAB}, {\tt GNU Octave}, {\tt Julia}, and {\tt python}.
We would like to recall that the product of quaternions is not commutative. Thus, special attention should be given to the order of the terms in the quaternion product. 

The conjugate and the norm of a quaternion $q=q_0+\vec{q}$, denoted respectively by $\bar{q}$ and $|q|$, are defined by 
\bb \label{eq:conjugate} \bar{q} = q_0-\vec{q} \qeq 
|q| = \sqrt{\bar{q} q} =  \sqrt{q_0^2 + q_1^2 + q_2^2 +q_3^2}. \ee
%Recall that $\overline{(\bar{q})}=q$, $\overline{(p+q)} = \bar{p} + \bar{q}$, and $\overline{(pq)} = \bar{q} \bar{p}$ hold true for any $p,q \in \mathbb{H}$.
We say that $q$ is a {\em unit} quaternion if $|q|=1$. We denote by $\mathbb{S}$ the set of all unit quaternions, i.e., $\mathbb{S} = \{q \in \mathbb{H}: |q|=1|\}$. Note the $\mathbb{S}$ can be regarded as an hypersphere in $\mathbb{R}^4$. The quanternion-valued function $\sigma:\mathbb{H}^* \to \mathbb{S}$ given by 
\bb \label{eq:sigma} \sigma(q) = \frac{q}{|q|}, \ee
maps the set of non-zero quaternions $\mathbb{H}^* = \mathbb{H} \setminus\{0\}$ to the set of all unit quaternions. The function $\sigma$ can be interpreted as a generalization of the signal function to unit quaternions. Furthermore, $\sigma$ generalizes to the quaternion domain the complex-valued activation function proposed by Aizenberg and Moraga \cite{aizenberg07}.

Finally, the inner product of two quaternion-valued column vectors $\vetx=[x_1,\ldots,x_n]^T \in \mathbb{H}^n$ and $\vety = [y_1,\ldots,y_n]^T \in \mathbb{H}^n$ is given by 
\bb \label{eq:inner} \sprod{\vetx}{\vety} = \sum_{i=1}^n \bar{y}_i x_i. \ee
Note that $\sprod{\vetx}{\vetx} = n$ for all unit quaternion-valued vectors $\vetx \in \mathbb{S}^n$. The Euclidean norm of a quaternion-valued vector $\vetx \in \mathbb{H}^n$ is defined by $\|\vetx\|_2 = \sqrt{\sprod{\vetx}{\vetx}}$.
% In particular, we have
% \bb \re{\sprod{\vetx}{\vety}} = \sum_{i=1}^n (y_{i0}x_{i0} + y_{i1}x_{i1} + y_{i2}x_{i2} + y_{i3}x_{i3}),\ee
% which corresponds to the inner product of the real-valued vectors of length $4n$ obtained by concatenating the 4-tuples $(x_{i0},x_{i1},x_{i2},x_{i3})$ and $(y_{i0},y_{i1},y_{i2},y_{i3})$, for $i=1,\ldots,n$. 

\section{Quaternion-Valued Hopfield Neural Networks} \label{sec:QHNN}

The famous {\em Hopfield neural network} (HNN) is a recurrent model which can be used to implement associative memories \cite{hopfield82}. Quaternion-valued versions of the Hopfield network, which generalize complex-valued models, have been extensively investigated in the past years \cite{valle18tnnls,valle14bracis,kobayashi16a,isokawa06,isokawa07,isokawa08a,isokawa12,osana12,kobayashi17a,kobayashi17b}. A comprehensive review on several types of {\em quaternionic HNN} (QHNN) can be found in \cite{isokawa13,valle18tnnls}. Briefly, the main difference between the several QHNN models resides in the activation function. 

In this paper, we consider a quaternion-valued activation function $\sigma$ given by \eqref{eq:sigma} whose output is obtained by normalizing its argument to length one  \cite{valle14bracis,kobayashi16a}. The resulting network can be implemented and analyzed more easily than the multistate quaternion-valued Hopfield neural network proposed by Isokawa et al. \cite{isokawa08b,isokawa13}. Furthermore, as far as we know, it is the unique version of the Hopfield network on unit quaternions that always yields a convergent sequence in the asynchronous update mode under the usual conditions on the synaptic weights \cite{valle18tnnls}. We would like to point out that, together with the QHNN on unit quaternions, the QHNN based on the twin-multistate activation function are the unique stable quaternion-valued Hopfield networks \cite{kobayashi17a}. The QHNN based on twin-multistate activation function, however, do not generalize the bipolar and complex-valued models, i,e., the bipolar and complex-valued models are not particular instances of the twin-multistate QHNN.

The QHNN is defined as follows: Let $w_{ij} \in \mathbb{H}$ denotes the $j$th quaternionic synaptic weight of the $i$th neuron of a network with $n$ neurons. Also, let the state of the QHNN at time $t$ be represented by a column quaternion-valued vector $\vetx(t) = [x_1(t),\ldots,x_n(t)]^T \in \mathbb{S}^n$, that is, the unit quaternion $x_i(t) = x_{i0}(t) + x_{i1}(t)\ii + x_{i2}(t) \jj + x_{i3}(t)\kk$ corresponds to the state of the $i$th neuron at time $t$. Given an initial state (or input vector) $\vetx(0) = [x_1,\ldots,x_n]^T \in \mathbb{S}^n$, the QHNNs defines recursively the sequence of quaternion-valued vectors $\vetx(0),\vetx(1),\vetx(2), \ldots$ by means of the equation
\bb \label{eq:update} x_j(t+1) = \begin{cases} \sigma\left(a_j(t)\right), & 0<|a_j(t)|<+\infty, \\ x_j(t), & \mbox{otherwise}, \end{cases} \ee
where
\bb \label{eq:hopfield} a_i(t)=\sum_{j=1}^n w_{ij} x_j(t), \ee
is the activation potential of the $i$th neuron at iteration $t$. In analogy with the traditional real-valued bipolar Hopfield network, the sequence produced by \eqref{eq:update} and \eqref{eq:hopfield} using asynchronous update mode is convergent for any initial state $\vetx(0) \in \mathbb{S}^n$ if the synaptic weights satisfy \cite{valle14bracis}:
\bb \label{eq:converg} w_{ij}=\bar{w}_{ji} \qeq w_{ii} \geq 0, \quad \forall i,j \in \{1,\ldots,n\}. \ee 
Here, the inequality $w_{ii} \geq 0$ means that $w_{ii}$ is a non-negative real number.
Moreover, the synaptic weights of a QHNN are usually determined using either the correlation or projection rule \cite{isokawa13}. Both correlation and projection rule yield synaptic weights that satisfy \eqref{eq:converg}.

Consider a fundamental memory set $\mathcal{U}=\{\vetu^1,\ldots,\vetu^p \}$, where each $\vetu^\xi = [u_1^\xi,\ldots,u_n^\xi]^T$ is a quaternion-valued column vector whose components $u_i^\xi=\quat{u^\xi_i}$ are unit quaternions. In the quaternionic version of the {\em correlation rule}, also called {\em Hebbian learning} \cite{isokawa13}, the synaptic weights are given by
\bb \label{eq:correlation} w_{ij}^c = \frac{1}{n} \sum_{\xi=1}^p u_i^\xi \bar{u}_j^\xi, \quad \forall i,j \in \{1,2,\ldots,n\}. \ee  
Unfortunately, such as the real-valued correlation recording recipe, the quaternionic correlation rule is subject to the cross-talk between the original vectors $\vetu^1,\ldots,\vetu^p$. In contrast, the {\em projection rule}, also known as the {\em generalized-inverse recording recipe}, is a non-local storage prescription that can suppress the cross-talk effect between the fundamental memories $\vetu^1,\ldots,\vetu^p$ \cite{kanter87}. Formally, in the projection rule the synaptic weights are defined by 
\bb \label{eq:projection} w_{ij}^p = \frac{1}{n} \sum_{\eta=1}^p \sum_{\xi=1}^p u_i^\eta c_{\eta \xi}^{-1} \bar{u}_j^\xi, \ee  
where $c^{-1}_{\eta\xi}$ denotes the $(\eta,\xi)$-entry of the quaternion-valued inverse of the matrix $C \in \mathbb{H}^{p \times p}$ given by
\bb c_{\eta\xi} = \frac{1}{n} \sum_{j=1}^n \bar{u}^\eta_j u_j^\xi =  \frac{1}{n}  \sprod{\vetu^\xi}{\vetu^\eta}, \quad \forall \mu,\nu \in \{1,\ldots,p\}. \ee
It is not hard to show that, if the matrix $C$ is invertible, then $\sum_{j=1}^n w_{ij}^p u_j^\xi = u_i^\xi$ for all $\xi = 1,\ldots,p$ and $i=1,\ldots,n$ \cite{isokawa13}. Therefore, all the fundamental memories are fixed points of the QHNN with the projection rule. On the downside, the projection rule requires the inversion of a $p\times p$ quaternion-valued matrix. 

\section{Quaternion-Valued Recurrent Correlation Neural Networks} \label{sec:QRCNNs}

Recurrent correlation neural networks (RCNNs), formerly known as recurrent correlation associative memories (RCAMs), have been introduced in 1991 by Chiueh and Goodman for the storage and recall of $n$-bit vectors \cite{chiueh91}. In contrast to the original Hopfield neural network which has a limited storage capacity, some RCNN models can reach the storage capacity of an ideal associative memory \cite{hassoun88}. Furthermore, certain RCNNs can be viewed as kernelized versions of the Hopfield network with Hebbian learning \cite{garcia04a,garcia04b,perfetti08}. Finally, RCNNs are closely related to the dense associative memory model introduced by Krotov and Hopfield to establish the duality between associative memories and deep learning \cite{krotov16,demircigil17}. 

The RCNNs have been generalized for the storage and recall of complex-valued and quaternion-valued vectors \cite{valle14nnB,valle18wcci}. In the following, we briefly review the quaternionic recurrent neural networks (QRCNNs). Precisely, to pave the way for the development of the new models introduced in the next section, let us derive the QRCNNs from the correlation-based QHNN described by \eqref{eq:update}, \eqref{eq:hopfield}, and \eqref{eq:correlation}.

Consider a fundamental memory set $\mathcal{U}=\{\vetu^1,\ldots,\vetu^p \} \subset \mathbb{S}^n$. Using the synaptic weights $w_{ij}^c$ given by \eqref{eq:correlation}, we conclude from \eqref{eq:hopfield} that the activation potential of the $i$th neuron at iteration $t$ of the correlation-based QHNN satisfies
\begin{align*}
    a_i(t) &= \sum_{j=1}^n w_{ij}^c x_j(t) = \sum_{j=1}^n \left[\frac{1}{n} \sum_{\xi=1}^p u_i^\xi \bar{u}_j^\xi \right] x_j(t) \\
    &= \sum_{\xi=1}^p u_i^\xi \left[\frac{1}{n} \sum_{j=1}^n \bar{u}_j^\xi x_j(t) \right] \\ &= \sum_{\xi=1}^p u_i^\xi \left[\frac{1}{n} \sprod{\vetx(t)}{\vetu^\xi} \right].
\end{align*} 
In words, the activation potential $a_i(t)$ is given by a weighted sum of $u_i^1,\ldots,u_i^p$. Moreover, the weights are proportional to the inner product between the current state $\vetx(t)$ and the fundamental memory $\vetu^\xi$. 

In the QRCNN, the activation potential $a_i(t)$ is also given by a weighted sum of $u_i^1,\ldots,u_i^p$. Following a reasoning similar to the ``kernel trick'' \cite{scholkopf02}, however, the weights are given by function of the real part of the inner product $\sprod{\vetx(t)}{\vetu^\xi}$. Precisely, let $f:[-1,1]\rightarrow \R$ be a (real-valued) continuous and monotone non-decreasing function referred to as the excitation function. Given a quaternionic input vector $\vetx(0) = [x_1(0),\ldots,x_N(0)]^T \in \mathbb{S}^N$, a QRCNN defines recursively a sequence $\{\vetx(t)\}_{t\geq 0}$ of quaternion-valued vectors by means of \eqref{eq:update} where the activation potential of the $i$th output neuron at time $t$ is given by
\bb \label{eq:QRCNN} a_i(t)=\sum_{\xi=1}^p w_\xi(t) u_i^\xi, \quad \forall i=1,\ldots,n,\ee
with
\bb \label{eq:weights} w_\xi(t)= f\left(\frac{1}{n}\re{\sprod{\vetx(t)}{\vetu^\xi}}\right), \quad \forall \xi \in 1,\ldots,p. \ee

Alternatively, the dynamic of a QRCNN can be described using a matrix-vector notation. Let $U = [\vetu^1,\ldots,\vetu^p] \in \mathbb{S}^{n \times p}$ be the matrix whose columns correspond to the fundamental memories, $U^*$ denote the conjugate transpose of $U$, and assume the functions $f$, $\sigma$, and $\re{\cdot}$ are evaluated in a component-wise manner. Given an initial state $\vetx(0)$, the dynamic of a QRCNN using synchronous update is described by the equations
\bb \label{eq:matweights} \vetw(t) = f\big(\re{U^* \vetx(t)}/n\big),\ee
and
\bb \label{eq:matQRCNN} \vetx(t+1) = \sigma \big(U \vetw(t) \big). \ee

Such as the original \cite{chiueh91} and the complex-valued RCNNs \cite{valle14nnB}, a QRCNN is implemented by the fully connected two layer neural network with $p$ hidden neurons shown in Figure \ref{fig:topology}a). The first layer evaluates $f$ at the real part of the inner product $\sprod{\vetx(t)}{\vetu^\xi}$ divided by $n$. Equivalently, the first layer is equipped with quaternion-valued neurons whose quaternionic synaptic weights are given by the matrix $U^*$ and the activation function is $\varphi(\cdot) = f(\re{\cdot})$. The output layer evaluates $\sigma$ at a weighted sum of $\vetu^1,\ldots,\vetu^p$. In other words, the synaptic weights of the output neurons are given by the quaternionic matrix $U$ and their activation function is $\sigma$. Note that the first layer encodes the quaternion-valued vector $\vetx(t)$ of length $n$ into a real-valued vector $\vetw(t)$ of length $p$. Similarly, the next quaternion-valued state vector $\vetx(t+1)$, produced by the output layer, corresponds to the decoded version of the $p$-dimensional real-valued vector $\vetw(t)$.
%%%%%%%%%%%%%%%%%%%%%%%%%%%%%%%%%%%%%%%%%%%%%%%%%%%%%
\begin{figure}
\begin{center}
$$ \xymatrix@R0pt@C2pt{
{\color{blue} a)}&&&& {\color{blue} U^*} && {\color{blue}f \big( \Re \{\cdot\}/n \big)} && {\color{blue} U} && {\color{blue} \sigma (\cdot)} \\
x_1(t) \ar[rr] && {\bigcirc} \ar[rrrrd] \ar[rrrrddd] \ar[rrrrddddddd] &&&&&&&&
{\bigcirc} \ar[rr] && x_1 (t+1)\\
&&&&&& {\bigcirc} \ar[rrrru] \ar[rrrrd] \ar[rrrrddd] \ar[rrrrddddddd] &&&& \\
x_2(t) \ar[rr]&& {\bigcirc} \ar[rrrru] \ar[rrrrd] \ar[rrrrddddd] &&&&&&&&
{\bigcirc} \ar[rr] && x_2 (t+1)\\
&&&&&& {\bigcirc} \ar[rrrruuu] \ar[rrrru] \ar[rrrrd] \ar[rrrrddddd] &&&&&&\\
x_3(t) \ar[rr]&& {\bigcirc} \ar[rrrruuu] \ar[rrrru] \ar[rrrrddd]  &&&&&&&&
{\bigcirc} \ar[rr] && x_3 (t+1)\\
&&&&&& \vdots &&&&&&\\
\vdots && \vdots &&&&&&&& \vdots && \vdots\\
&&&&&& {\bigcirc} \ar[rrrruuuuuuu] \ar[rrrruuuuu] \ar[rrrruuu] \ar[rrrrd] &&&&\\
x_n(t) \ar[rr]&& {\bigcirc} \ar[rrrruuuuuuu] \ar[rrrruuuuu] \ar[rrrru] &&&&&&&&
{\bigcirc} \ar[rr] && x_n (t+1) \\
{\color{red} b)}&&&& {\color{red} U^*} && {\color{red}f \big( \Re \{\cdot\}/n \big)} && {\color{red} V} && {\color{red} \sigma (\cdot)}}$$
\caption{The network topology of quaternionic recurrent {\color{blue} correlation} and {\color{red} projection} neural networks.} \label{fig:topology}
\end{center}
\end{figure}
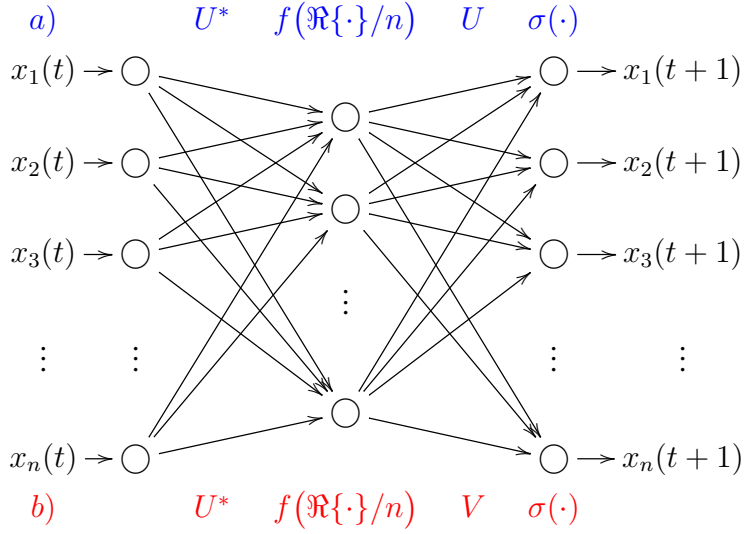
%%%%%%%%%%%%%%%%%%%%%%%%%%%%%%%%%%%%%%%%%%%%%%%%%%%%%

Examples of QRCNNs include the following straightforward quaternionic generalizations of the bipolar RCNNs:
\begin{enumerate}
 \item The {\em identity QRCNN} is obtained by considering in \eqref{eq:weights} the identity function $f_i(x)=x$. 
 \item The {\em high-order QRCNN}, which is determined by the function \bb \label{eq:f_h} f_h(x;q)=(1+x)^q, \quad q>1. \ee
 \item The {\em potential-function QRCNN}, which is obtained by considering in \eqref{eq:weights} the function 
 \bb \label{eq:f_p} f_p(x;L)=\frac{1}{(1-x+\varepsilon_p)^L}, \quad L \geq 1, \ee 
 where $\varepsilon_p>0$ is a small number\footnote{Like our previous works \cite{valle14nnB,valle18wcci}, we adopted the value $\varepsilon_p=\sqrt{\epsilon_{mach}}$, where $\epsilon_{mach}$ denotes the machine floating-point relative accuracy, in our computational implementation $f_p$.} introduced to avoid a division by zero when $x=1$.  
 \item The {\em exponential QRCNN}, which is determined by an exponential \bb \label{eq:f_e} f_e(x;\alpha)= e^{\alpha x}, \quad \alpha>0. \ee
\end{enumerate}

Note that QRCNNs generalize both bipolar and complex-valued RCNNs \cite{chiueh91,valle14nnB}. Precisely, the bipolar and the complex-valued models are obtained by considering vectors $\vetx=[x_1,\ldots,x_n]^T \in \mathbb{S}^n$ whose components satisfy respectively $x_j = {x_j}_0 + 0\ii + 0\jj + 0\kk$ and $x_j = {x_j}_0+{x_j}_1\ii+0\jj+0\kk$ for all $j=1,\ldots,n$. Furthermore, the identity QRCNN generalizes the traditional bipolar correlation-based Hopfield neural network but it does not generalize the correlation-based QHNN. Indeed, in contrast to the correlation-based QHNN, the identity QRCNN uses only the real part of the inner product $\sprod{\vetx(t)}{\vetu^\xi}$. 

Let us briefly turn our attention to the high-order, potential-function, and the exponential excitation functions. Fixed the first argument $x$, the functions $f_h$, $f_p$, and $f_e$ are all exponential with respect to their parameters. Precisely, these three excitation functions belong to the following family of parametric functions:
\begin{align} \label{eq:familyF} \mathcal{F} = \{& f(x;\lambda): f(x;\lambda) = [A(x)]^\lambda \; \mbox{for} \; \lambda\geq 0 \; \mbox{and a continuous function $A$} \\ \nonumber &\qquad  \mbox{such that} \; 0 \leq A(x_1) < A(x_2) \; \mbox{when} \; -1 \leq x_1 < x_2 \leq 1\}.\end{align}
The family $\mathcal{F}$ is particularly important for the implementation of associative memories  \cite{valle18wcci}. Precisely, if the excitation function belongs to the family $\mathcal{F}$, then a QRCNN can reach the storage capacity of an ideal associative memory by choosing a sufficiently large parameter $\lambda$. On the downside, overflow imposes an upper bound on the parameter $\lambda$ of an excitation function $f \in \mathcal{F}$ which may the limit application of a QRCNN as associative memory.  

Finally, we would like to point out that,  independently of the initial state $\vetx(0) \in \mathbb{S}^n$ and the update mode (synchronous or asynchronous), a QRCNN model always yields a convergent sequence $\{\vetx(t)\}_{t \geq 0}$ \cite{valle18wcci}. Therefore, QRCNNs are potential models to implement associative memories. Morevoer, due to the non-linearity of the activation functions of the hidden neurons, the high-order, potential-function, and exponential QRCNNs may overcome the rotational invariance problem found on quaternionic Hopfield neural network \cite{kobayashi16a}. On the downside, such as the correlation-based quaternionic Hopfield network, QRCNNs may suffer from cross-talk between the fundamental memories $\vetu^1,\ldots,\vetu^p$. Inspired by the projection rule, the next section introduces improved models which overcome the cross-talk problem of the QRCNNs.

\section{Quaternion-Valued Recurrent Projection Neural Networks}
\label{sec:QRPNNs}

Quaternion-valued recurrent projection neural networks (QRPNNs) combine the main idea behind the projection rule and the QRCNN models to yield high capacity associative memories. Specifically, using the synaptic weights $w_{ij}^p$ given by \eqref{eq:projection}, the activation potential of the $i$th neuron at time $t$ of the projection-based QHNN is
\begin{align*}
    a_i(t) &= \sum_{j=1}^n w_{ij}^p x_j(t) 
    = \sum_{j=1}^n \left[\frac{1}{n} \sum_{\eta=1}^p \sum_{\xi=1}^p u_i^\eta c_{\eta \xi}^{-1} \bar{u}_j^\xi \right] x_j(t) \\
    & = \sum_{\eta=1}^p \sum_{\xi=1}^p u_i^\eta  c_{\eta \xi}^{-1} \left[\frac{1}{n} \sum_{j=1}^n \bar{u}_j^\xi x_j(t) \right] \\
    & = \sum_{\xi=1}^p \left( \sum_{\eta=1}^p u_i^\eta c_{\eta \xi}^{-1} \right) \left[\frac{1}{n} \sprod{\vetx(t)}{\vetu^\xi} \right].
\end{align*} 
In analogy to the QRCNN, we replace the term proportional to the inner product between $\vetx(t)$ and $\vetu^\eta$ by the weight $w_\xi(t)$ given by \eqref{eq:weights}. Accordingly, we define $c^{-1}_{\eta \xi}$ as the $(\eta,\xi)$-entry of the inverse of the real-valued matrix $C \in \mathbb{R}^{p \times p}$ given by
\bb \label{eq:Cmatrix} c_{\eta\xi} = f \left( \frac{1}{n}  \re{\sprod{\vetu^\xi}{\vetu^\eta}} \right), \quad \forall \eta,\xi \in \{1,\ldots,p\}.  \ee
Furthermore, to simplify the computation, we define 
\bb \label{eq:vs} v_i^\xi = \sum_{\eta=1}^p u_i^\eta c_{\eta \xi}^{-1},\ee 
for all $i=1,\ldots, n$ and $\eta = 1,\ldots,p$. Thus, the activation potential of a QRPNN is given by
\bb \label{eq:QRPNN} a_i(t)=\sum_{\xi=1}^p w_\xi(t) v_i^\xi, \quad \forall i=1,\ldots,n.\ee

Concluding, given a fundamental memory set $\mathcal{U} = \{\vetu^1,\ldots,\vetu^p\} \subset \mathbb{S}^n$, define the $p \times p$ real-valued matrix $C$ by means of \eqref{eq:Cmatrix} and compute the quaternion-valued vectors $\vetv^1,\ldots,\vetv^p$ using \eqref{eq:vs}. Like the QRCNN, given an input vector $\vetx(0)\in \mathbb{S}^n$, a QRPNN yields the sequence $\{\vetx(t)\}_{t\geq 0}$ by means of \eqref{eq:update} where the activation potential of the $i$th output neuron at time $t$ is given by \eqref{eq:QRPNN} with $w_\xi(t)$ defined by \eqref{eq:weights}. 

Alternatively, using a matrix-vector notation, a synchronous QRPNN can be described as follows: Let $U = [\vetu^1,\ldots,\vetu^p] \in \mathbb{S}^{n \times p}$ be the quaternion-valued matrix whose columns corresponds to the fundamental memories. Define the real-valued matrix $C \in \mathbb{R}^{p \times p}$ and the quaternion-valued matrix $V \in \mathbb{H}^{n \times p}$ by means of the equations
\bb \label{eq:CandV} C= f(U^* U/n) \quad \mbox{and} \quad V = U C^{-1},\ee
where the excitation function $f$ is evaluated in an entry-wise manner. Note that the two equations in \eqref{eq:CandV} are equivalent to \eqref{eq:Cmatrix} and \eqref{eq:vs}, respectively. Given the initial state $\vetx(0)$, a QRCNN defines recursively
\bb \vetw(t) = f\big(\re{U^* \vetx(t)}/n\big), \ee and \bb \vetx(t+1) = \sigma\big(V \vetw(t)\big), \ee
where $f:[-1,1] \to \R$ and $\sigma:\mathbb{H}^* \to \mathbb{S}$ are evaluated in a component-wise manner. Algorithm \ref{alg:QRPNN}, formulated using matrix notation, summarizes the implementation of a QRPNN using synchronous update. We would like to point out that a QRCNN is obtained by setting $V=U$ in Algorithm \ref{alg:QRPNN}. 

\begin{algorithm}[t] \label{alg:QRPNN}
\KwData{\begin{enumerate}
    \item A continuous and non-decreasing real-valued function $f:\R \to \R$.
    \item Matrices $U = [\vetu^1,\ldots,\vetu^p]$ and $V = [\vetv^1,\ldots,\vetv^p]$.
    \item The input vector $\vetx = [x_1,\ldots,x_n]$.
    \item Maximum number of iterations $t_{\max}$ and a tolerance $\tau>0$.
\end{enumerate}
}
\KwResult{Retrieved vector $\vety$.}

\BlankLine
% \If{ $V = [\vetv^1,\ldots,\vetv^p]$ is not provided as input}{
% 1. Define the $p \times p$ real-valued matrix $C$ by means of the equation
% \[C = f \left( \frac{1}{n}  \re{U^* U} \right).\]

% 2. Compute $C^{-1}$, the inverse of the matrix $C$.

% 3. Define the matrix $V$ as follows:
% \[ V = U C^{-1}.\]
% }
% \BlankLine

Initialize $t=0$ and $\Delta=\tau+1$.
\BlankLine
\While{$t \leq t_{\max}$ \textbf{and} $\Delta \geq \tau$}{
1. Compute the weights \[w_\xi= f\left(\frac{1}{n}\re{U^* \vetx} \right). \]

2. Compute the next state 
\[ \vety = \sigma(V \vetw).\]

3. Update respectively $t \leftarrow t+1$, $\Delta \leftarrow \|\vety-\vetx\|$, and $\vetx \leftarrow \vety$.
}
\caption{Quaternion-valued Recurrent Projection Neural Network}
\end{algorithm}

Like the QRCNN, a QRPNN is also implemented by the fully connected two layer neural network with $p$ hidden neurons shown in Figure \ref{fig:topology}b). The difference between the QRCNN and the QRPNN is the synaptic weight matrix of the output layer. In other words, they differ in the way the real-valued vector $\vetw(t)$ is decoded to yield the next state $\vetx(t)$. 

From the computational point of view, although the training phase of a QRPNN requires $\mathcal{O}(p^3 + n p^2)$ operations to compute the matrices $C^{-1}$ and $V$, they usually exhibit better noise tolerance than the QRCNNs. Moreover, the following theorem shows that QRCNNs overcome the cross-talk between the fundamental memories if the matrix $C$ is invertible. Precisely, the following
theorem shows that all the fundamental memories are stationary states of a QRPNN if the matrix $C$ is invertible.

\begin{theorem} \label{thm:fixed_points}
Given a fundamental memory set $\mathcal{U}=\{\vetu,\ldots,\vetu^p\}$, define the real-valued $p \times p$-matrix $C$ by \eqref{eq:Cmatrix}. If $C$ is invertible, then all fundamental memories $\vetu^1,\ldots,\vetu^p$ are stationary states of a QRPNN defined by \eqref{eq:update}, \eqref{eq:weights}, and \eqref{eq:QRCNN}. 
\end{theorem}

\begin{proof}
Let us assume that the matrix $C$ given by \eqref{eq:Cmatrix} is invertible. Also, suppose a QRPNN is initialized at a fundamental memory, that is, $\vetx(0)=\vetu^\gamma$ for some $\gamma \in \{1,\ldots,p\}$. From \eqref{eq:weights} and \eqref{eq:Cmatrix}, we conclude that 
\[ w_\xi(0) = f \big( \re{\sprod{\vetu^\gamma}{\vetu^\xi}}/n \big) = c_{\xi \gamma}, \quad \forall \xi =1,\ldots,p.\]
Furthermore, from \eqref{eq:QRPNN} and \eqref{eq:vs}, we obtain the following identities for any $i \in \{1,\ldots,n\}$:
\begin{align*}
    a_i(0) &= \sum_{\xi=1}^p w_\xi(0) v_i^\xi 
    = \sum_{\xi=1}^p \left( \sum_{\eta=1}^p u_i^\eta c_{\eta \xi}^{-1} \right) c_{\xi \gamma} \\
    &= \sum_{\eta=1}^p u_i^\eta \left( \sum_{\xi=1}^p  c_{\eta \xi}^{-1} c_{\xi \gamma}  \right) 
    = \sum_{\eta=1}^p u_i^\eta \delta_{\eta \gamma} = u^\gamma_i,
\end{align*}
where $\delta_{\eta\gamma}$ is the Kronecker delta, that is, $\delta_{\eta \gamma} = 1$ if $\eta = \gamma$ and $\delta_{\eta \gamma}=0$ if $\eta \neq \gamma$. From \eqref{eq:update}, we conclude that the fundamental memory $\vetu^\gamma$ is a fixed point of the QRPNN if the matrix $C$ is invertible.
\end{proof}

Let us investigate further the relationship between QRPNNs and QRCNNs. Theorem \ref{thm:fixed_points} shows that QRPNNs can be used to implement an associative memory whenever the matrix $C$ is invertible. It turns out that QRCNNs can also be used to implement an associative memory using an excitation function $f \in \mathcal{F}$ with a sufficiently large parameter $\lambda$ \cite{valle18wcci}. Assuming $f \in \mathcal{F}$, the following theorem shows that the matrix $C$ given by \eqref{eq:Cmatrix} is invertible if the parameter $\lambda$ is sufficiently large. Moreover, the QRPNN and the QRCNN coincide in this case. 

\begin{theorem} \label{thm:RPNNxRCNN}
Consider a fundamental memory set $\mathcal{U} =\{\vetu^1,\ldots,\vetu^p\} \subset \mathbb{S}^n$ and  
an excitation function $f \in \mathcal{F}$, where $\mathcal{F}$ is the family of parameterized functions given by \eqref{eq:familyF}. The matrix $C$ given by \eqref{eq:Cmatrix} is invertible for a parameter $\lambda$ sufficiently large. Furthermore, given an arbitrary state vector $\vetx \in \mathbb{S}^n$, let $\vetx_P$ and $\vetx_C$ denote respectively the states of the QRPNN and QRCNN models after one single synchronous update, that is, 
\bb \vetx_P = \sigma( V \vetw) \qeq \vetx_C = \sigma(U \vetw), \ee 
where $\vetw = f(\re{U^*\vetx}/n;\lambda)$. In this case, $\vetx_P$ approaches $\vetx_C$ as $\lambda$ tends to infinity. Formally, we have
\bb \label{eq:limit} \lim_{\lambda \to \infty} \| \vetx_P - \vetx_C \|_2 = 0, \ee 
where $\|\cdot\|_2 = \sqrt{\sprod{\vetx}{\vetx}}$ denotes the Euclidean norm.
\end{theorem}

\begin{proof}
First of all, recall that $f(x;\lambda) = [A(x)]^\lambda$ because $f \in \mathcal{F}$. Let $A_1 = A(1)>0$ denote the maximum value of the function $A$. Also, note that 
\[ \re{\sprod{\vetx}{\vety}} = n - \frac{1}{2} \| \vetx -\vety \|_2^2, \quad \forall \vetx,\vety \in \mathbb{S}^n.\]
As a consequence, $\re{\sprod{\vetx}{\vety}}/n = 1$ if and only if $\vetx = \vety$ and $\re{\sprod{\vetx}{\vety}}<1$ if $\vetx \neq \vety$. Therefore, an entry of the matrix $C$ given by \eqref{eq:Cmatrix} satisfies
\begin{align*}
\lim_{\lambda \to \infty} \frac{c_{\eta \xi}}{A_1^\lambda} 
&= \lim_{\lambda \to \infty} \frac{1}{A_1^\lambda} f \left(\frac{1}{n}\re{\sprod{\vetu^\xi}{\vetu^\eta}}; \lambda \right) \\
&=  \lim_{\lambda \to \infty} \left[ \frac{A\left(\re{\sprod{\vetu^\xi}{\vetu^\eta}}/n \right)}{A_1} \right]^\lambda \\ 
&= \begin{cases}
1, & \xi = \eta, \\
0, & \mbox{otherwise.}
\end{cases}    
\end{align*}
Equivalently, 
\bb \label{eq:inverseC} 
\lim_{\lambda \to \infty} \frac{1}{A_1^\lambda}C =
\lim_{\lambda \to \infty} \left(\frac{1}{A_1^\lambda}I \right) C = I,\ee
where $I$ denotes the identity matrix. In a similar fashion, we conclude that
\bb \label{eq:limit_w} \lim_{\lambda \to \infty} \frac{\vetw}{A_1^\lambda} = \begin{cases}
\boldsymbol{e}_\xi, & \vetx = \vetu^\xi, \\
0, & \mbox{otherwise},
\end{cases} \ee 
where $\boldsymbol{e}_\xi$ denotes the $\xi$th column of the $p \times p$ identity matrix. Since the matrix product is continuous and the inverse of $C$ is unique, $C^{-1}$ exists and approaches $A_1^{-\lambda} I$ as $\lambda$ increases. In other words, \eqref{eq:inverseC} ensures the existence of $C^{-1}$ for $\lambda$ sufficiently large.

Let us now show \eqref{eq:limit}. To this end, let us define $\vetz = C^{-1} \vetw$ or, equivalently, $\vetw = C \vetz$. From \eqref{eq:inverseC} and \eqref{eq:limit_w}, we conclude that
\begin{align*}
\lim_{\lambda \to \infty} \vetz 
&= \left[\lim_{\lambda \to \infty} \left(\frac{1}{A_1^\lambda}C \right)\right] \left[ \lim_{\lambda \to \infty} \vetz \right] 
= \lim_{\lambda \to \infty} \left(\frac{1}{A_1^\lambda}C \right)\vetz 
= \lim_{\lambda \to \infty} \frac{1}{A_1^\lambda}(C \vetz) \\ 
& = \lim_{\lambda \to \infty} \frac{\vetw}{A_1^\lambda}
= \boldsymbol{e}_\xi 
.\end{align*}
%
% Since $\vetx_P$, $\vetx_C$, $\vetw$, and $C^{-1}$ depend on the parameter $\lambda$, to avoid confusion, we write $\vetx_P(\lambda)$, $\vetx_C(\lambda)$, $\vetw(\lambda)$, and $C^{-1}(\lambda)$. Also, let us define $\vetz(\lambda) = C^{-1}(\lambda) \vetw(\lambda)$ or, equivalently, $\vetw(\lambda) = C(\lambda) \vetz(\lambda)$. From \eqref{eq:inverseC} and \eqref{eq:limt_w}, assuming that $\lim_{\lambda \to \infty} \vetz(\lambda)$ exists, we obtain
%
Now, recalling that $V=UC^{-1}(\lambda)$, $\|\cdot\|$ and $\sigma$ are continuous, and $A_1 > 0$, we obtain
\begin{align*}
& \lim_{\lambda \to \infty} \|\vetx_P - \vetx_C\|_2
= \lim_{\lambda \to \infty} \| \sigma\big(V\vetw\big) - \sigma\big(U\vetw\big)\|_2 \\
&= \lim_{\lambda \to \infty} \left\| \sigma\big(U C^{-1}\vetw\big) - \sigma\left(\frac{1}{A_1^\lambda} U  \vetw\right) \right\|_2 \\
&= \left\| \sigma\left(U \left(\lim_{\lambda \to \infty} C^{-1} \vetw \right)\right) - \sigma\left(U\left(\lim_{\lambda \to \infty} \frac{1}{A_1^\lambda}\vetw\right) \right) \right \|_2 \\
&= \left\| \sigma\left(U \left(\lim_{\lambda \to \infty} \vetz \right)\right) - \sigma\left(U\left(\lim_{\lambda \to \infty} \frac{1}{A_1^\lambda}\vetw\right) \right) \right \|_2 \\
&= \left\| \sigma\left(U \boldsymbol{e}_\xi \right) - \sigma\left(U\boldsymbol{e}_\xi  \right) \right \|_2 
=0.
\end{align*}
The last identity concludes the proof of the theorem.
\end{proof}

We would like to point out that the basic idea behind Theorem \ref{thm:RPNNxRCNN} is that the matrix $C$ given by \eqref{eq:Cmatrix} can be approximated by a multiple of the identity matrix for a sufficiently large parameter $\lambda$ of an excitation function $f \in \mathcal{F}$. Borrowing the terminology from \cite{perfetti08}, we say that a QRCNN as well as a QRPNN are in saturated mode if the matrix $C$ can be approximated by $c I$ for some $c>0$. In the saturated mode, QRCNNs and QRPNNs coincide. 

% \begin{remark}
% Theorem \ref{thm:RPNNxRCNN} shows that the matrix $C$ given by \eqref{eq:Cmatrix} is invertible if $\lambda$ is sufficiently large. Therefore, like the QRCNNs, an QRPNN based on an excitation function $f \in \mathcal{F}$ can implement an associative memory by selecting a parameter $\lambda$ sufficiently large. In practice, however, the parameter $\lambda$ is bounded by the floating point number system or the dynamic range in VLSI realizations. 
% \end{remark}

In analogy to the QRCNNs, the identity mapping and the functions $f_h$, $f_p$, and $f_e$ given by \eqref{eq:f_h}, \eqref{eq:f_p}, and \eqref{eq:f_e} are used to define respectively the {\em identity QRPNN}, the {\em high-order QRPNN}, the {\em potential-function QRPNN}, and the {\em exponential QRPNN}. 

Note that the identity QRPNN generalizes the traditional bipolar projection-based Hopfield neural network. The identity QRPNN, however, does not generalize the projection-based QHNN because the former uses only the real part of the inner product between $\vetx(t)$ and $\vetu^\xi$. In fact, in contrast to the projection-based QHNN, the design of a QRPNN does not require the inversion of a quarternion-valued matrix but only the inversion of a real-valued matrix. 

\subsection{Bipolar RPNNs and Recurrent Kernel Associative Memories} \label{sec:RKAM}

As pointed out previously, quaternion-valued RPNNs reduce to bipolar models when the fundamental memories are all real-valued, that is, their vector part is zero. In this subsection, we address the relationship between bipolar RPNNs and the recurrent kernel associative memories (RKAMs) proposed by Garcia and Moreno \cite{garcia04a,garcia04b} and further investigated by Perfetti and Ricci \cite{perfetti08}. 
% To this end, let us assume that the fundamental memories as well as the states of the networks are bipolar represented by bipolar vectors, that is, vectors in $\{-1,+1\}$. Recall that bipolar RPNNs are particular instances of the quaternion-valued RPNNs introduced previously.

A RKAM model is defined as follows. Let $\kappa$ denote an inner-product kernel and $\rho>0$ be a user-defined parameter. Given a fundamental memory set $\mathcal{U}=\{\vetu^1,\ldots,\vetu^p\} \subseteq \{-1,+1\}^n$, define the Lagrange multipliers vector $ \betav_{i} = [ \beta_{i1},\ldots, \beta_{ip}]$ as the solution of the following quadratic problem for $i=1,\ldots,n$:
\bb  \label{eq:QPwBC}
\begin{cases}
\mbox{minimize} & \displaystyle{Q( \betav_i) =  \frac{1}{2}\sum_{\xi,\eta = 1}^p  \beta_i^\xi  \beta_i^\eta u_i^\xi u_i^\eta \kappa(\vetu^\xi, \vetu^\eta)} - \sum_{\xi=1}^p  \beta_{i}^{\xi},\\ 
\mbox{subject to} & \displaystyle{0 \leq  \beta_{i}^{\xi} \leq \rho}, \; \forall \xi=1,\ldots,p.
\end{cases}
\ee 
Then, given a bipolar initial state $\vetx(0) \in \{-1,+1\}^n$, a RKAM evolves according to the following equation for all $i=1,\ldots,n$:
\bb \label{eq:RKAM} x_i(t+1) = \sgn\left(\sum_{\xi=1}^p  \beta_{i}^{\xi} u_i^\xi \kappa\big(\vetu^\xi,\vetx(t)\big) \right). \ee

Note that $x_i(t+1)$, the next state of the $i$th neuron of a RKAM, corresponds to the output of a support vector machine (SVM) classifier without the bias term determined using the training set $\mathcal{T}_i = \{(\vetu^\xi,u_i^\xi): \xi=1,\ldots,p\}$. Therefore, the design of a RKAM requires, in some sense, training $n$ independent support vector classifiers (one SVM for each output neuron of the RKAM!). Furthermore, like the soft-margin SVM, the used-defined parameter $\rho$ controls the trade-off between the training error and the separation margin \cite{scholkopf02}. In the associative memory context, the larger the parameter $\rho$ is, the larger the storage capacity of the RKAM. Conversely, some fundamental memories may fail to be stationary states of the RKAM if $\rho$ is small.

Let us now compare the RKAMs with bipolar RCNNs and RPNNs. To this end, let us assume the excitation function $f$ is a valid kernel. The exponential function $f_e$, for example, yields a valid kernel, namely the Gaussian radial-basis function kernel \cite{perfetti08}. As pointed out by Perfetti and Ricci \cite{perfetti08}, the main difference between a RKAM and a RCAM is the presence of the Lagrange multipliers in the former. Precisely, the Lagrange multipliers are $ \beta_i^\xi=1$ in the RCNNs while, in the RKAM, they are obtained solving \eqref{eq:QPwBC}. Furthermore, the RKAM is equivalent to the corresponding bipolar RCNN in the saturation model, that is, when the matrix $C$ given by \eqref{eq:Cmatrix} exhibits a diagonal dominance \cite{perfetti08}. 
In a similar fashion, we observe that a RKAM and a RPNN coincide if $v_i^\xi =  \beta_i^\xi u_i^\xi$ for all $i=1,\ldots,n$ and $\xi=1,\ldots,p$. The following theorem shows that this equation holds true if all the constraints are inactive at the solution of the quadratic problem \eqref{eq:QPwBC}.

\begin{theorem} \label{thm:RKAMxRPNN}
Let $f:[-1,+1]\to \R$ be a continuous and non-decreasing function such that the following equation yields a valid kernel
\bb \label{eq:f2kernel} \kappa(\vetx,\vety) = f\left(\frac{ \sprod{\vetx}{\vety} }{n}\right), \quad \forall \vetx,\vety \in \{-1,+1\}^n.\ee
Consider a fundamental memory set $\mathcal{U}=\{\vetu^1,\ldots,\vetu^p\} \subset \{-1,+1\}^n$ such that the matrix $C$ given by \eqref{eq:Cmatrix} is invertible. If the solutions of the quadratic problem defined by \eqref{eq:QPwBC} satisfy $0< \beta_i^\xi<C$ for all $i=1,\ldots,n$ and $\xi=1,\ldots,p$, then the RKAM defined by \eqref{eq:QPwBC} and \eqref{eq:RKAM} coincide with the RPNN defined by \eqref{eq:update}, \eqref{eq:weights}, and \eqref{eq:QRPNN}. Alternatively, the RKAM and the bipolar RPNN coincide if the vectors $\vetv^\xi$'s given by \eqref{eq:vs} satisfy $0<v_i^\xi u_i^\xi < \rho$ for all $i=1,\ldots,n$ and $\xi=1,\ldots,p$.
\end{theorem}

\begin{proof}
First of all, note that $\kappa(\vetu^\xi,\vetu^\eta) = c_{\xi\eta}$ given by \eqref{eq:Cmatrix}. 

Let us first show that the RKAM coincide with the bipolar RPNN if the inequalities $0 <  \beta_i^\xi <C$ hold true for all $i$ and $\xi$.
If there is no active constrain at the solution of \eqref{eq:QPwBC}, then the Lagrange multiplier $\betav_i$ are also the solution of the unconstrained quadratic problem \eqref{eq:QPwBC}:
\bb \label{eq:Qalpha} \mbox{minimize} \; Q( \betav_i) = \frac{1}{2}\sum_{\xi,\eta = 1}^p   \beta_i^\xi u_i^\xi c_{\xi\eta} u_i^\eta  \beta_i^\eta  - \sum_{\xi=1}^p  \beta_{i}^{\xi}, \quad \forall i=1,\ldots,n.\ee
It turns out that the minimum of \eqref{eq:Qalpha}, obtained by imposing $\frac{\partial Q}{\partial \beta_i^\xi} = 0$, is the solution of the linear system 
\bb \label{eq:Qalpha2} \sum_{\eta = 1}^p u_i^\xi c_{\xi \eta} u_i^\eta  \beta_i^\eta  = 1, \quad \xi=1,\ldots,p.\ee
Multiplying \eqref{eq:Qalpha2} by $u_i^\xi$ and recalling that $(u_i^\xi)^2=1$, we obtain
\bb \label{eq:Qalpha4} \sum_{\eta = 1}^p c_{\xi \eta} u_i^\eta  \beta_i^\eta  = u_i^\xi, \quad \xi=1,\ldots,p.\ee
Since the matrix $C$ is invertible, the solution of the linear system of equations \eqref{eq:Qalpha4} is 
\bb \label{eq:Qalpha3} u_i^\xi  \beta_i^\xi  = \sum_{\eta = 1}^p c_{\xi \eta}^{-1} u_i^\eta, \quad \forall \xi=1,\ldots,p \qeq i=1,\ldots,n,\ee
where $c_{\xi \eta}^{-1}$ denotes the $(\xi,\eta)$-entry of $C^{-1}$. We conclude the first part of the proof by noting that the right-hand side of equations \eqref{eq:Qalpha3} and \eqref{eq:vs} coincide. Therefore, $v_i^\xi =  \beta_i^\xi u_i^\xi$ for all $i=1,\ldots,n$ and $\xi=1,\ldots,p$ and the RKAM coincides with the bipolar QRPNN.

On the other hand, if $0< v_i^\xi u_i^\xi < \rho$ for all $i=1,\ldots,n$ and $\xi=1,\ldots,p$, then $v_i^\xi =  \beta_i^\xi u_i^\xi$ is a solution of \eqref{eq:Qalpha3}. Equivalently, $ \beta_i^\xi = u_i^\xi v_i^\xi$ is the solution of the unconstrained quadratic problem \eqref{eq:Qalpha} as well as the quadratic problem with bounded constraints \eqref{eq:QPwBC}. As a consequence, the RKAM defined by \eqref{eq:QPwBC} and \eqref{eq:RKAM} coincide with the RPNN defined by \eqref{eq:update}, \eqref{eq:weights}, and \eqref{eq:QRPNN}.
\end{proof}

We would like to point out that the condition $0<\beta_i^\xi<\rho$ often occurs in the saturation mode, that is, when the matrix $C$ exhibits a diagonal dominance. In particular, the matrix $C$ is diagonally dominant given a sufficiently large parameter $\lambda$ of an excitation function $f \in \mathcal{F}$. In fact, given a parametric function $f \in \mathcal{F}$, the Lagrange multiplier $\beta_i^\xi$ approaches $1/f(1;\lambda)$ as $\lambda$ increases \cite{perfetti08}. Concluding, in the saturation mode, the RKAM, RCNN, and RPNN are all equivalent. We shall confirm this remark in the computational experiments presented in the next section.   

\section{Computational Experiments} \label{sec:ComputationalExperiments}

This section provides computational experiments comparing the performance of QHNNs, QRCNNs and the new QRPNN models as associative memory models. Although the paper address quaternionic models, let us begin by addressing the noise tolerance and storage capacity of the recurrent neural networks for the storage and recall of bipolar real-valued vectors. The bipolar case are also used to confirm the results presented in Section \ref{sec:RKAM}. We address the noise tolerance and storage capacity of quaternion-valued vectors subsequently. We would like to point out that the source-codes of the computational experiments, implemented in \texttt{Julia Language}, are available \url{https://github.com/mevalle/Quaternion-valued-Recurrent-Projection-Neural-Networks}. 

\subsection{Bipolar Associative Memories}  \label{ssec:Bipolar}
Let us compare the storage capacity and noise tolerance of the Hopfield neural networks (HNNs), the original RCNNs, and the new RPNNs designed for the storage of $p=36$ randomly generated bipolar (real-valued) vectors of length $n=100$. Precisely, we consider the correlation-based and projection-based Hopfield neural networks, the identity, high-order, potential-function, and exponential RCNN and RPNN models with parameters $q=5$, $L=3$, and $\alpha=4$, respectively. 

To evaluate the storage capacity and noise tolerance of the bipolar associative memories, the following steps have been performed 100 times for $n=100$ and $p=36$: 
\begin{enumerate}
\item We synthesized associative memories designed for the storage and recall of a randomly generated fundamental memory set $\mathcal{U}=\{\vetu^1,\ldots,\vetu^p\} \subset \{-1,+1\}^n$, where $\mbox{Pr}[u_i^\xi = 1] = \mbox{Pr}[u_i^\xi = -1]  = 0.5$  for all $i=1,\ldots,n$ and $\xi=1,\ldots,p$. 
\item We probed the associative memories with an input vector $\vetx(0)=[x_1(0),\ldots,x_n(0)]^T$ obtained by reversing some components of $\vetu^1$ with probability $\pi$, i.e., $\mbox{Pr}[x_i(0) = -u_i^1] = \pi$ and $\mbox{Pr}[x_i(0) = u_i^1] = 1-\pi$, for all $i$ and $\xi$. 
\item The associative memories have been iterated until they reached a stationary state or completed a maximum of 1000 iterations. A memory model succeeded to recall a stored item if the output equals $\vetu^1$. 

\end{enumerate}
\begin{figure}[t]
    \centering
    \includegraphics[width=1\columnwidth]{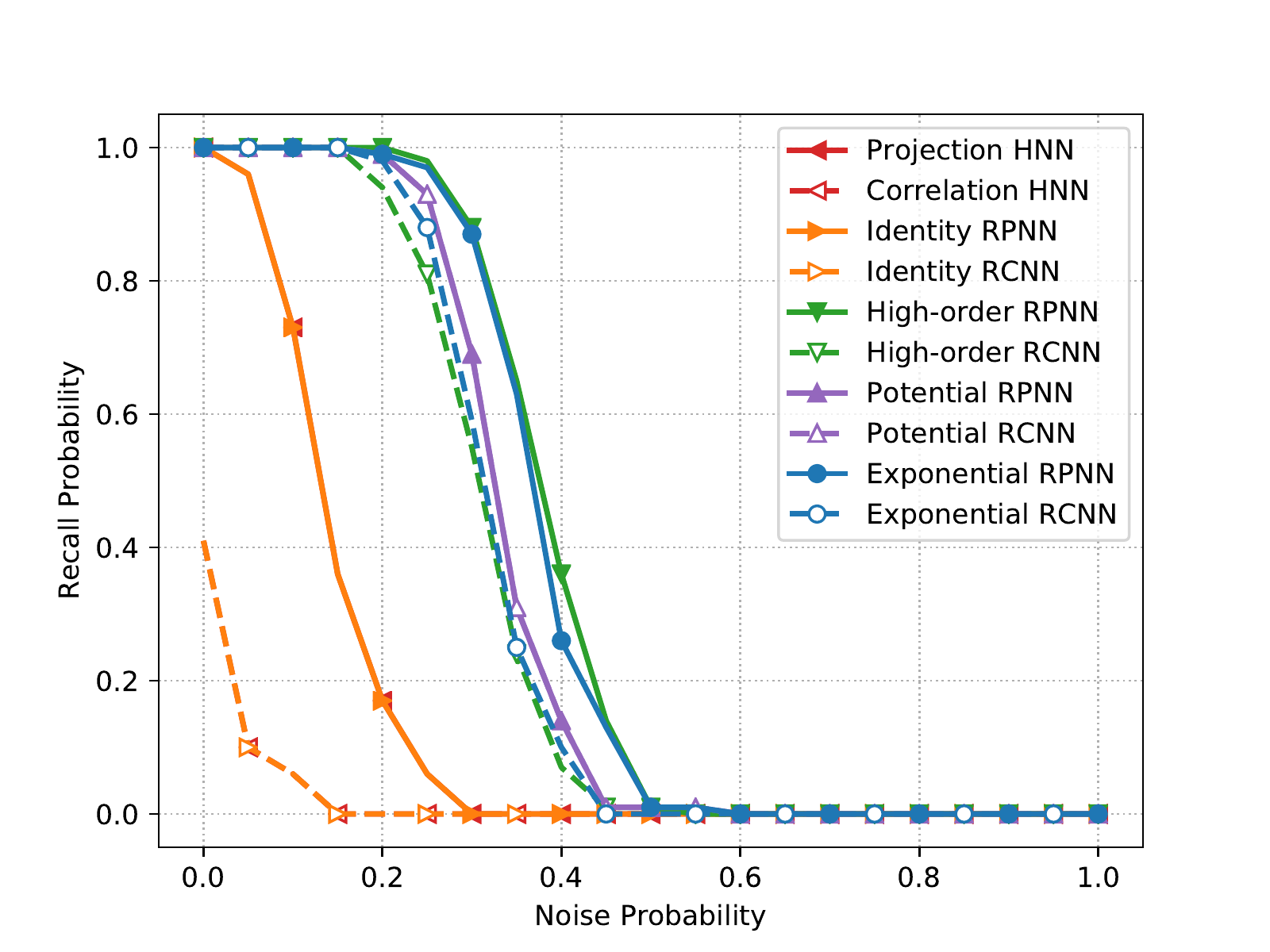}
    \caption{Recall probability of bipolar associative memories by the noise intensity introduced in the input vector.}
    \label{fig:Bipolar}
\end{figure}
Figure \ref{fig:Bipolar} shows the probability of an associative memory to recall a fundamental memory by the probability of noise introduced in the initial state. Note that the projection-based HNN coincides with the identity RPNN. Similarly, the correlation-based HNN coincides with the identity RCNN. Also, note that the RPNNs always succeeded to recall undistorted fundamental memories (zero noise probability). The high-order, potential-function, and exponential RCNNs also succeeded to recall undistorted fundamental memories. Nevertheless, the recall probability of the high-order and exponential RPNNs are greater than or equal to the recall probability of the corresponding RCNNs. In other words, the RPNNs exhibit better noise tolerance than the corresponding RCNNs. The potential-function RCNN and RPNN yielded similar recall probabilities.

\begin{figure}[t]
    \centering
    \includegraphics[width=1\columnwidth]{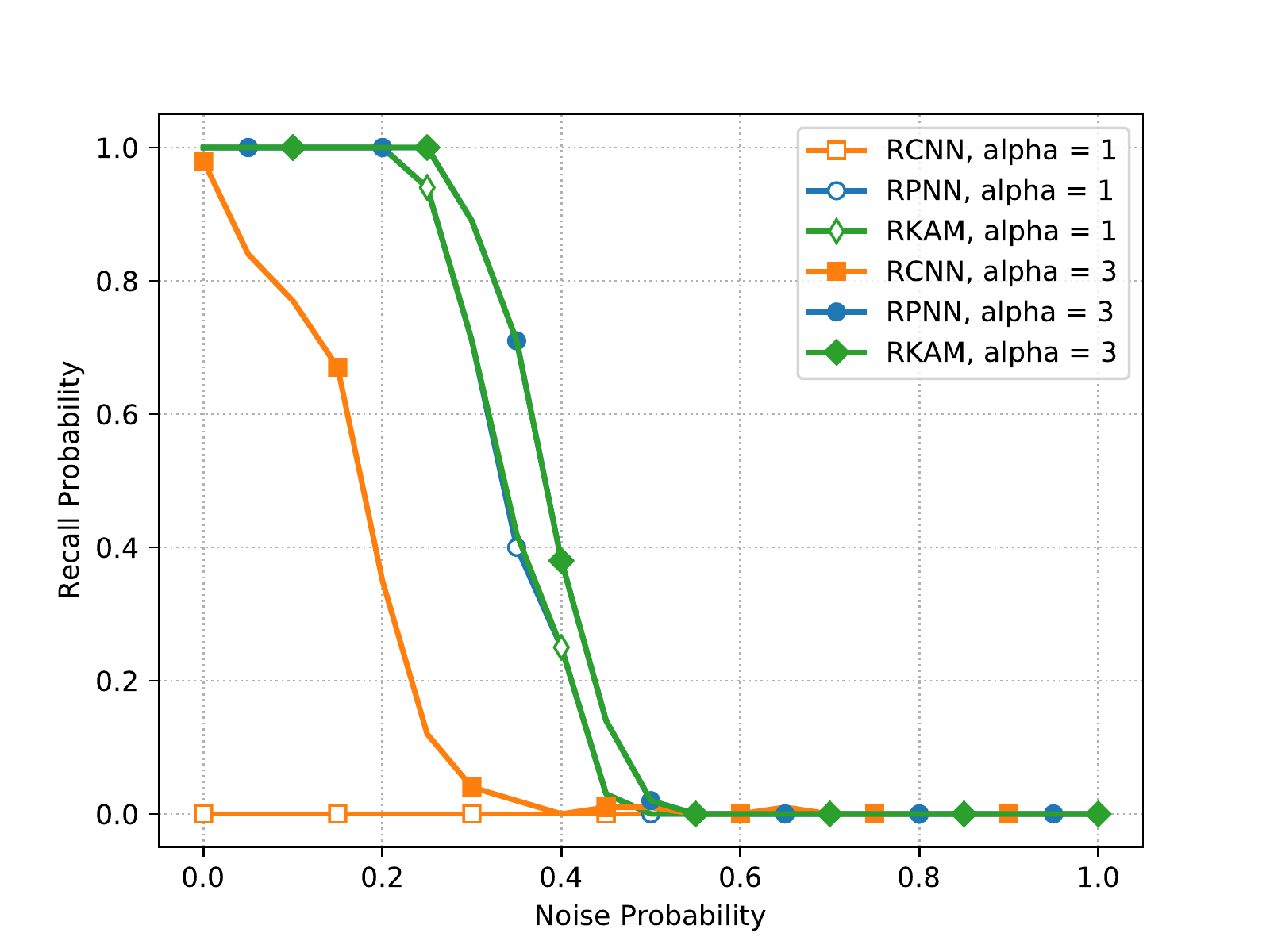}
    \caption{Recall probability of exponential bipolar RPNN and RKAM by the noise intensity introduced in the input vector.}
    \label{fig:RKAM_Exponential}
\end{figure}
In a similar fashion, let us compare the storage capacity and noise tolerance of the RCNN, RPNN and the RKAM with kernel given by \eqref{eq:f2kernel} with $f\equiv f_e$. In this experiment, we considered $\rho = 1000$ and $\alpha=1$ and $\alpha=3$. Figure \ref{fig:RKAM_Exponential} shows the recall probabilities of the exponential bipolar RCNN and RKAM, with different parameter values, by the noise probability introduced in the input. Note that both RPNN and RKAM outperformed the RCNN. Furthermore, the exponential bipolar RPNN and the RKAM coincided for all the values of the parameter $\alpha \in \{1,3\}$. According to Theorem \ref{thm:RKAMxRPNN}, an RKAM coincide with a bipolar RPNN if the Lagrange multipliers satisfy $0<\beta_i^\xi<\rho$ for all $i=1,\ldots,n$ and $\xi=1,\ldots,p$. Figure \ref{fig:RKAM_ExpHist} shows the histogram of the Lagrange multipliers obtained solving \eqref{eq:QPwBC} for a random generated matrix $U$ and the exponential kernel with $\alpha=1$, 2, and 3. The vertical dotted lines correspond to the values $e^{-3}$, $e^{-2}$, and $e^{-1}$. Note that $\beta_i^\xi$ approaches $1/f(1)=e^{-\alpha}$ as $\alpha$ increases.   More importantly, the inequalities $0<\beta_i^\xi<\rho$ are satisfied for $\alpha>1$ and $\rho \geq 2$.
\begin{figure}[t]
    \centering
    \includegraphics[width=1\columnwidth]{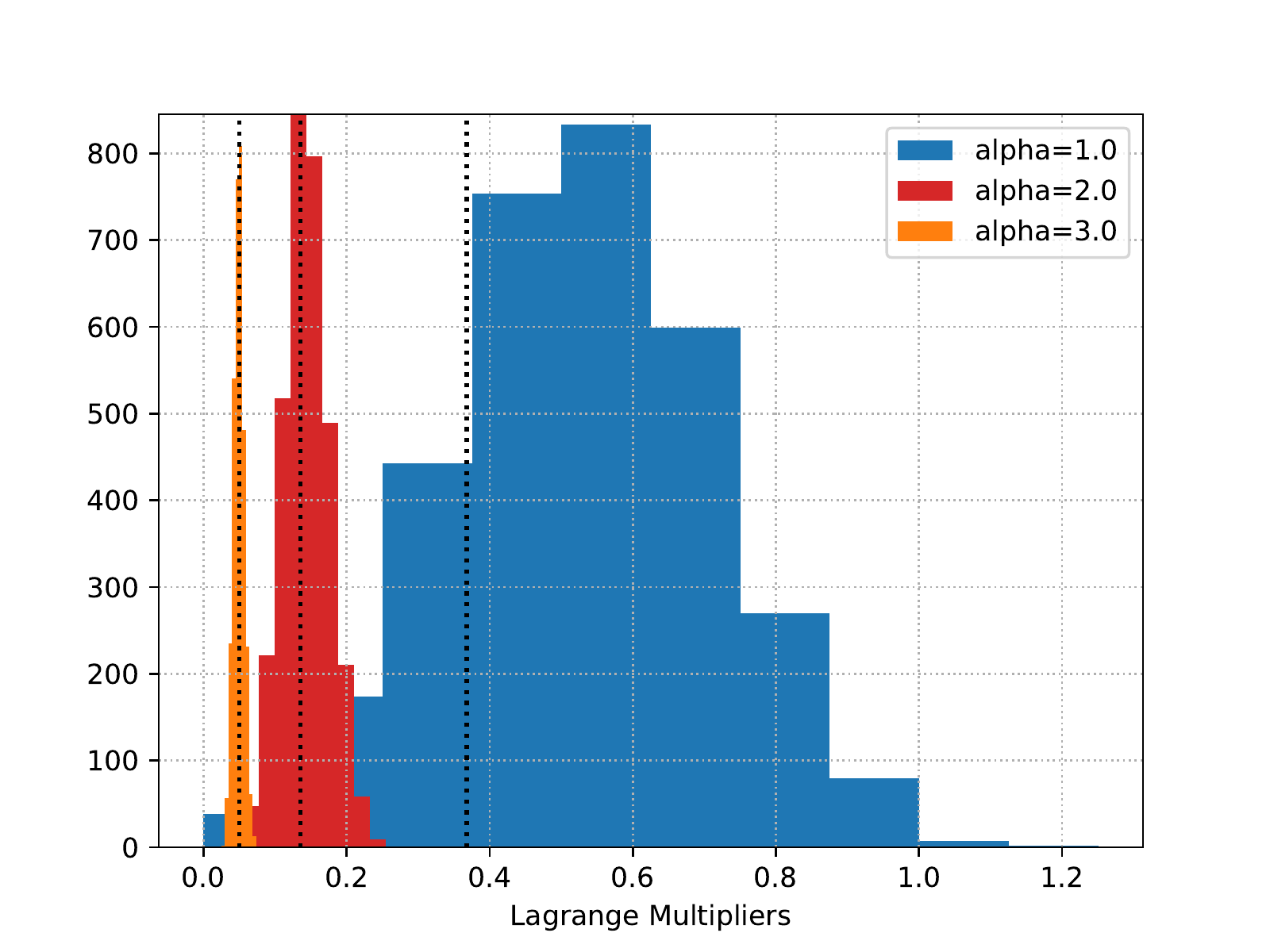}
    \caption{Histogram of the Lagrange multipliers of a RKAM model.}
    \label{fig:RKAM_ExpHist}
\end{figure}

\subsection{Quaternion-valued Associative Memories} \label{ex:Quaternion}

Let us now investigate the storage capacity and noise tolerance of the asssociative memory models for the storage and recall of $p=36$ randomly generated quaternion-valued vectors of length $n=100$. In this example, we considered the projection-based and the correlation-based quaternion-valued Hopfield neural network (QHNNs) as well as the identity, high-order, potential-function, and exponential QRCNNs and QRPNNs with parameters $q=20$, $L=3$, and $\alpha=15$. These parameters have been determined so that the QRCNNs have more than 50\% probability to recall an undistorted fundamental memory.
In analogy to the previous example, the following steps have been performed 100 times: 
\begin{enumerate}
\item We synthesized associative memories designed for the storage and recall of uniformly distributed fundamental memories $\mathcal{U}=\{\vetu^1,\ldots,\vetu^p\}$. Formally, we defined $u_i^\xi = \mathtt{RandQ}$ for all $i=1,\ldots,n$ and $\xi=1,\ldots,p$ where
\[ \mathtt{RandQ} = (\cos \phi+\ii\sin\phi)(\cos\psi+ \kk\sin \psi)(\cos \theta+\jj \sin \theta), \]
is a randomly generated unit quaternion obtained by sampling angles $\phi \in [-\pi,\pi)$, $\psi \in [-\pi/4,\pi/4]$, and $\theta \in [-\pi/2,\pi/2)$ using an uniform distribution. 
\item We probed the associative memories with an input vector $\vetx(0)=[x_1(0),\ldots,x_n(0)]^T$ obtained by replacing some components of $\vetu^1$ with probability $\pi$ by an uniformly distributed component, i.e., $\mbox{Pr}[x_i(0) = \mathtt{RandQ}] = \pi$ and $\mbox{Pr}[x_i(0) = u_i^1] = 1-\pi$, for all $i$ and $\xi$.
\item The associative memories have been iterated until they reached a stationary state or completed a maximum of 1000 iterations. The memory model succeeded if the output equals the fundamental memory $\vetu^1$.
\end{enumerate} 

\begin{figure}[t]
    \centering
    \includegraphics[width=1\columnwidth]{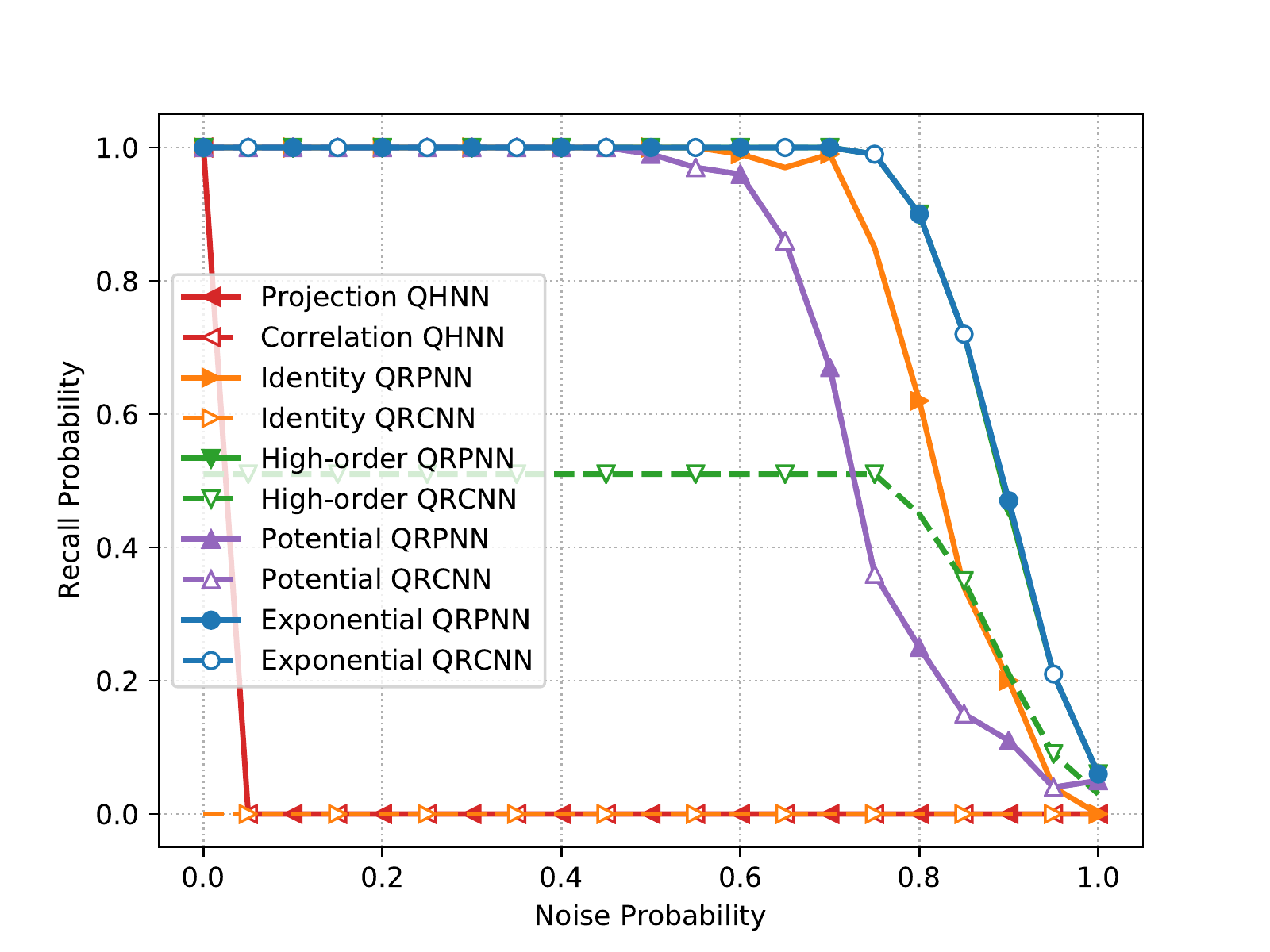}
    \caption{Recall probability of quaternion-valued associative memories by the noise intensity introduced in the input vector.}
    \label{fig:Quaternion}
\end{figure}
Figure \ref{fig:Quaternion} shows the probability of a quaternion-valued associative memory to recall a fundamental memory by the probability of noise introduced in the initial state. 
As expected, the QRPNNs always succeeded to recall undistorted fundamental memories. The potential-function and exponential QRCNNs also succeeded to recall undistorted fundamental memories. Indeed, the potential-function QRCNN and QRPNNs yielded the same recall probability. The noise tolerance of the exponential QRCNN and QRPNN also coincided. Nevertheless, the recall probability of the high-order QRPNN is greater than the recall probability of the corresponding QRCNNs. Furthermore, in contrast to the real-valued case, the projection QHNN differs from the identity QRPNN. In fact, the noise tolerance of the identity QRPNN  is far greater than the noise tolerance of the projection QHNN.

\subsection{Storage and Recall of Color Images}

In the previous subsection, we compared the performance of QHNN, QRCNN, and QRPNN models designed for the storage and recall of uniformly distributed fundamental memories. Let us now compare the performance of the quaternion-valued associative memories for the storage and recall of color images. Specifically, let us compare the noise tolerance of QHNN, QRCNN, and QRPNN models when the input is an color image corrupted by Gaussian noise. Recall that Gaussian noise are introduced in a color image, for example, due to faulty sensors \cite{plataniotis99}. Computationally, an image corrupted by Gaussian noise is obtained by adding a term drawn from an Gaussian distribution with zero mean and a fixed standard variation to each channel of a color image. 

At this point, we would like to recall that an RGB color image $\imI$ can be converted to a unit quaternion-valued vector $\vetx=[x_1,\ldots,x_n] \in \mathbb{S}^n$, $n=1024$, as follows \cite{castro17bracis}: Let $I_i^R \in [0,1]$, $I_i^G \in [0,1]$, and $I_i^B \in [0,1]$ denote respectively the red, green, and blue intensities at the $i$th pixel of an RGB color image $\imI$. For $i=1,\ldots,n$, we first compute the phase-angles 
\begin{align}
\label{eq:phi}    \phi_i &= \left(-\pi +\epsilon \right) + 2\left(\pi - \epsilon \right) I_i^R, \\
\label{eq:psi}    \psi_i &= \left(-\frac{\pi}{4} + \epsilon \right) + \left( \frac{\pi}{2}-2\epsilon \right) I_i^G, \\
\label{eq:theta}    \theta_i &= \left(-\frac{\pi}{2} + \epsilon \right) + \left( \pi-2\epsilon \right) I_i^G,
\end{align}
where $\epsilon>0$ is a small number such that $\phi_i \in [-\pi,\pi)$, $\psi_i \in [-\pi/4,\pi/4]$, and $\theta_i \in [-\pi/2,\pi/2)$  \cite{buelow99}. In our computational experiments, we adopted $\epsilon = 10^{-4}$. Then, we define the unit quaternion-valued vector $\vetx$ using the phase-angle representation $x_i = e^{\phi_i \ii} e^{\psi_i \kk} e^{\theta_i \jj}$ of its components. Equivalently, we have
\bb x_i = (\cos\phi_i \ii \sin\phi_i)(\cos \psi_i + \kk \sin \psi_i)(\cos\theta_i+\jj\sin\theta_i), \quad \forall i=1,\ldots,n.\ee 
Conversely, given an unit quaternion-valued vector $\vetx$, we first compute the phase-angles $\phi_i$, $\psi_i$, and $\theta_i$ of the component $x_i$ using Table 2.2 in \cite{buelow99}. Afterwards, we obtain the RGB color image $\imI$ by inverting \eqref{eq:phi}-\eqref{eq:theta}, that is,
\begin{align}
    I_i^R = \frac{\phi_i+\pi-\epsilon}{2*(\pi-\epsilon)}, \quad
    I_i^G = \frac{\psi_i+\pi/4-\epsilon}{(\pi/2-2\epsilon)}, \qeq
    I_i^B = \frac{\theta_i+\pi/2-\epsilon}{(\pi-2\epsilon)},
\end{align}
for all $i=1,\ldots,n$.

\begin{figure}[t]
    \begin{tabular}{cccc}
    {\footnotesize a) Original image} &  {\footnotesize b) Corrupted image} &
    \parbox{0.22\columnwidth}{\centering \footnotesize c) Correlation-based QHNN} & \parbox{0.22\columnwidth}{\centering \footnotesize d) Projection-based QHNN} \\
\includegraphics[width=0.22\columnwidth]{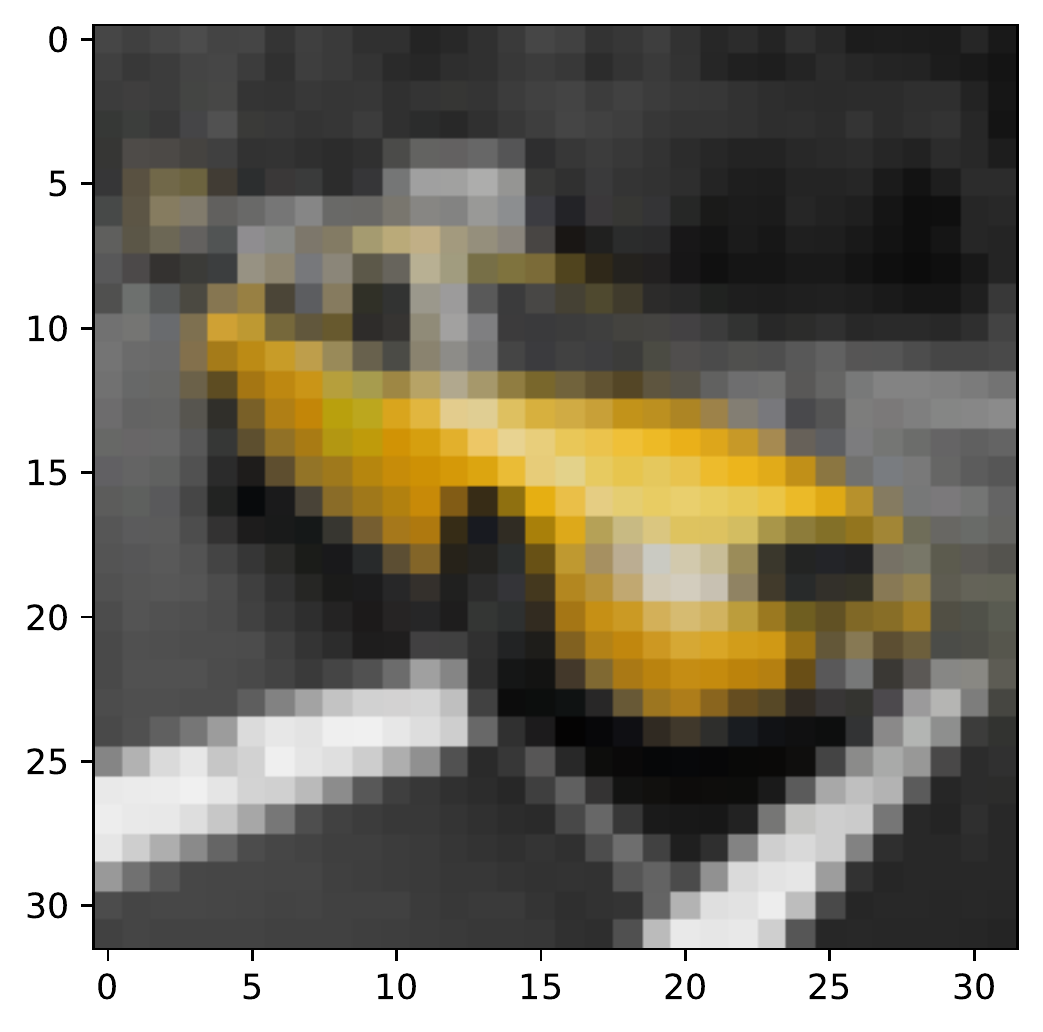} &
\includegraphics[width=0.22\columnwidth]{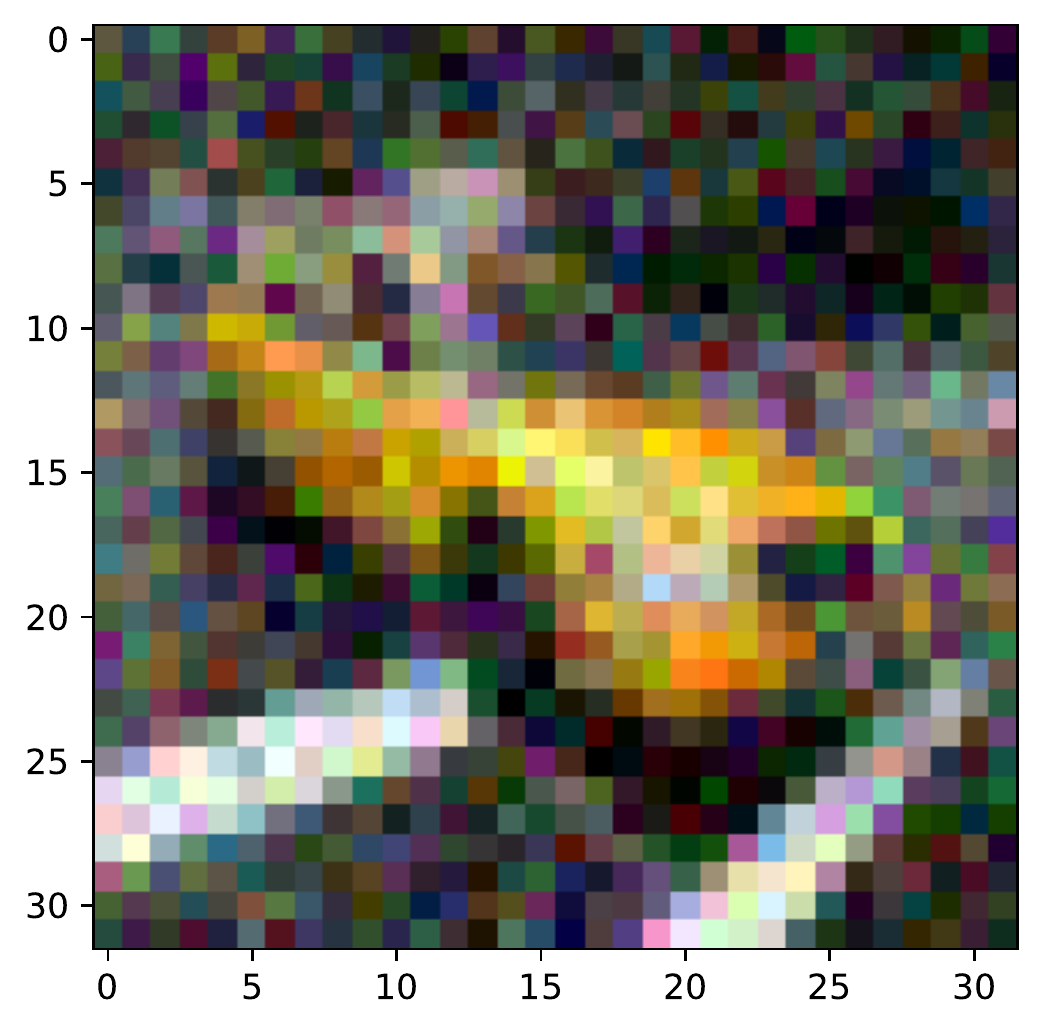} &
\includegraphics[width=0.22\columnwidth]{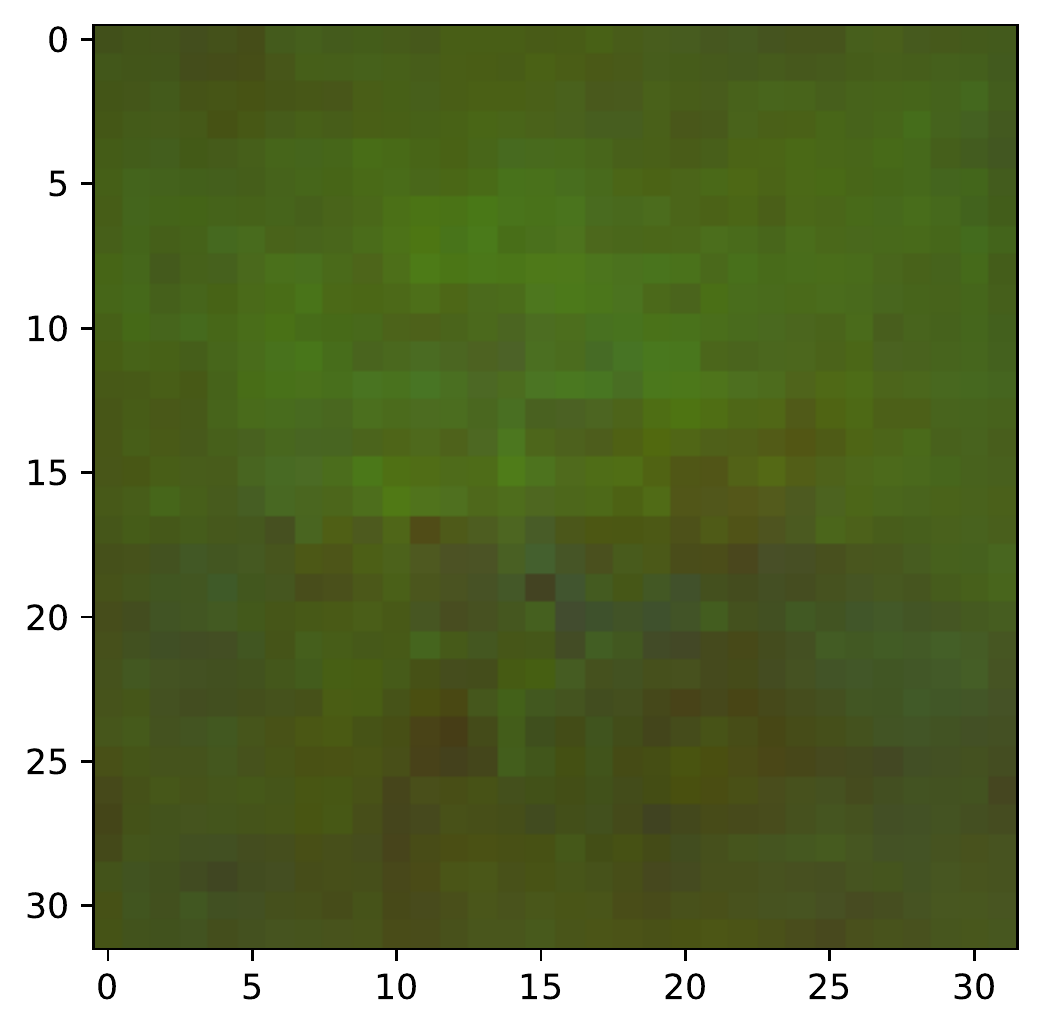} &
\includegraphics[width=0.22\columnwidth]{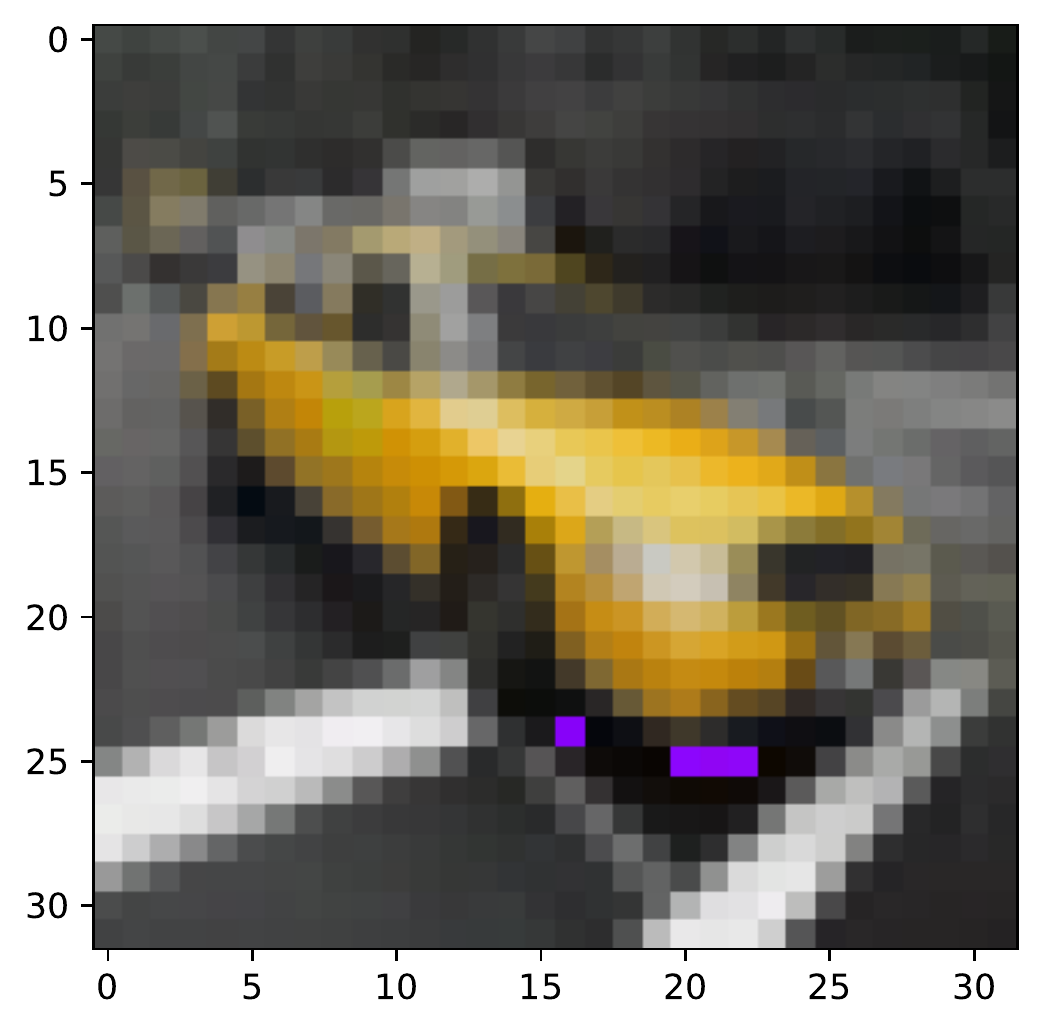} \\
%%%%%%%%
%%%%%%%%%%%%%%%%%%%%%%%%%%%%%%%%%%%%%%%%%%%%%%%%
{\footnotesize e) Identity QRCNN} &  
\parbox{0.22\columnwidth}{\centering \footnotesize f) High-order QRCNN ($q=70$)} &
    \parbox{0.22\columnwidth}{\centering \footnotesize g) Potential-function QRCNN ($L=5$)} &  \parbox{0.22\columnwidth}{\centering \footnotesize h) Exponential QRCNN ($\alpha=40$)}  \\
\includegraphics[width=0.22\columnwidth]{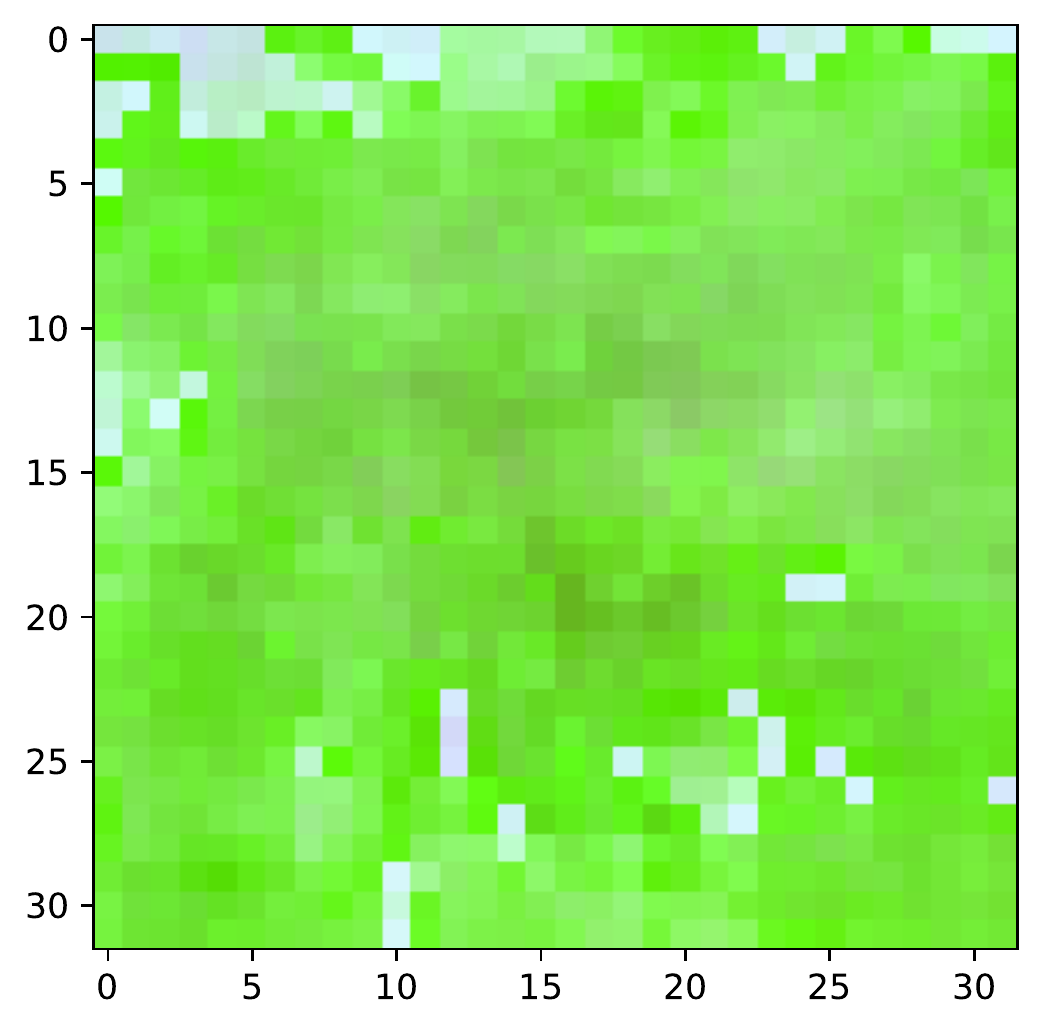} &
\includegraphics[width=0.22\columnwidth]{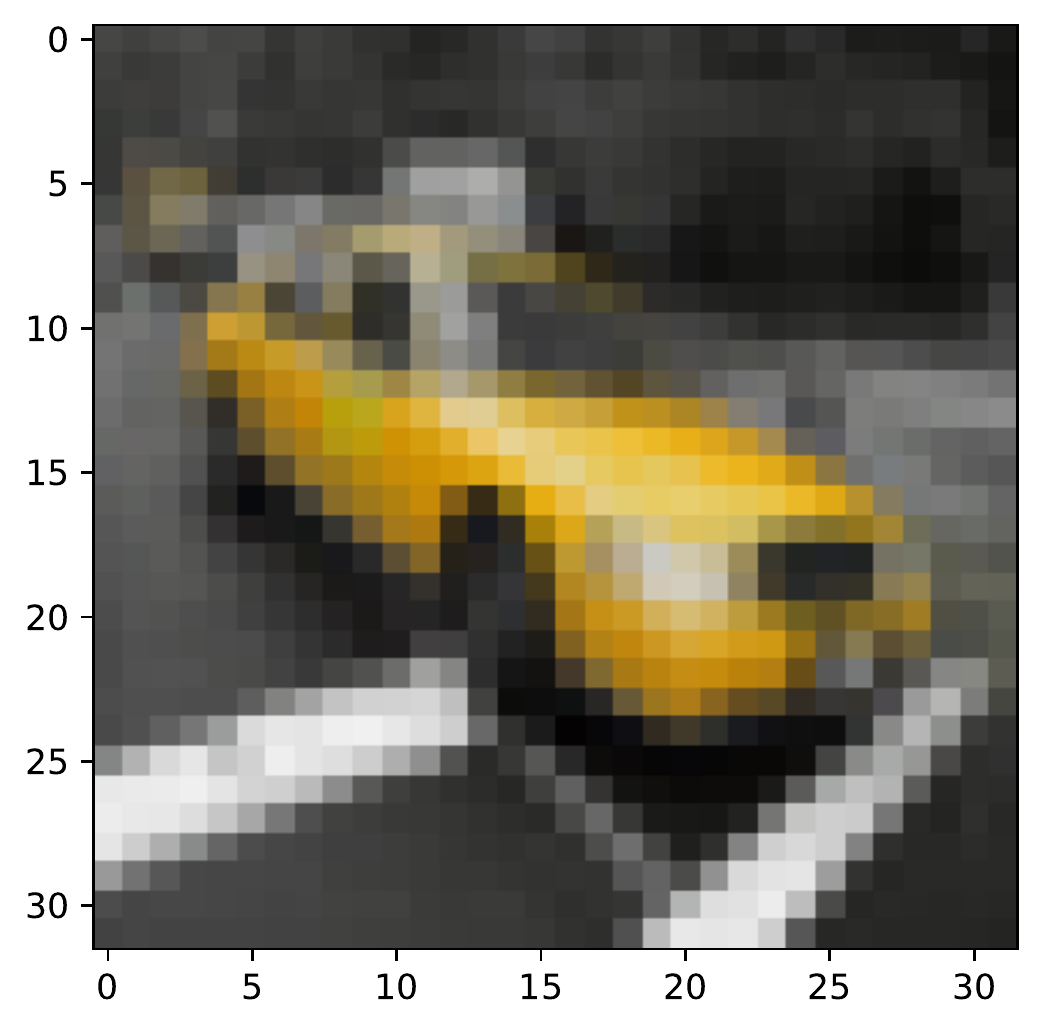} &
\includegraphics[width=0.22\columnwidth]{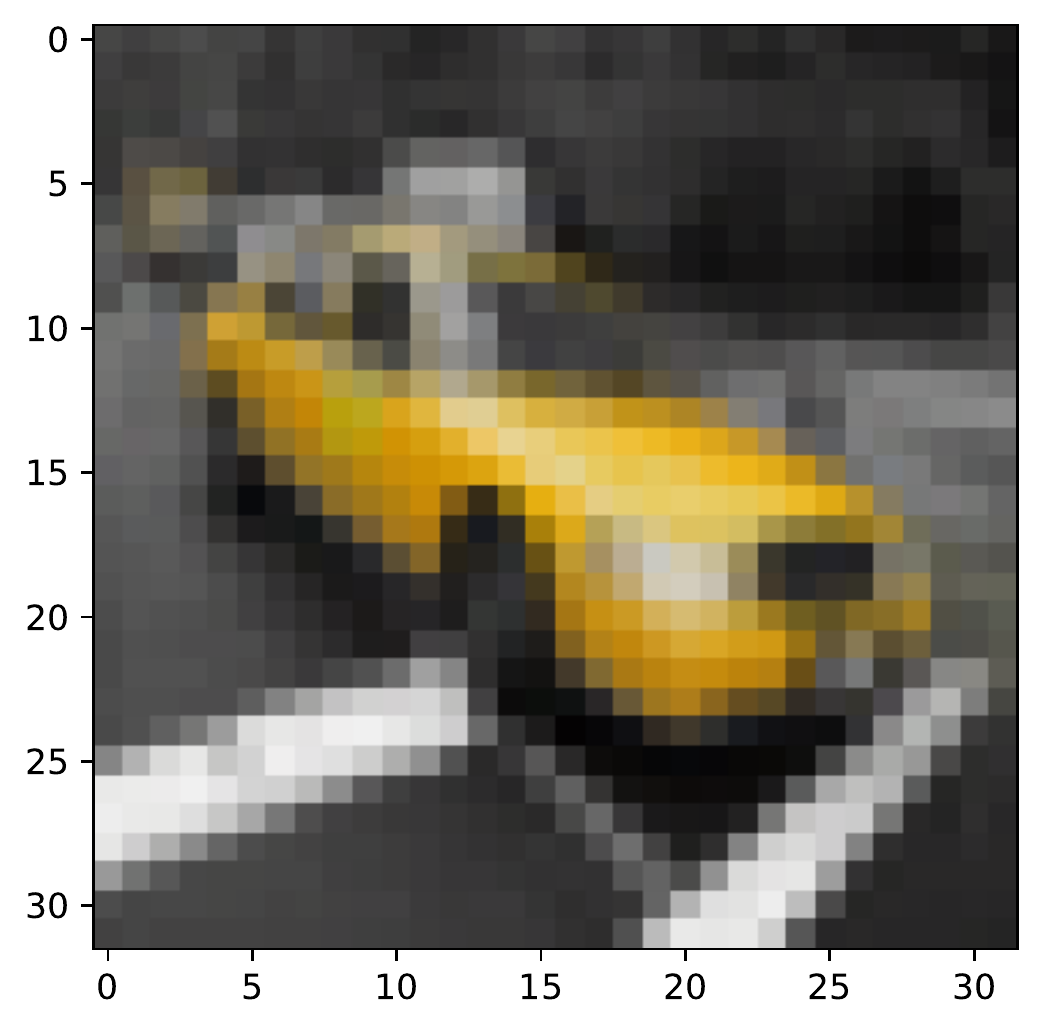} &
\includegraphics[width=0.22\columnwidth]{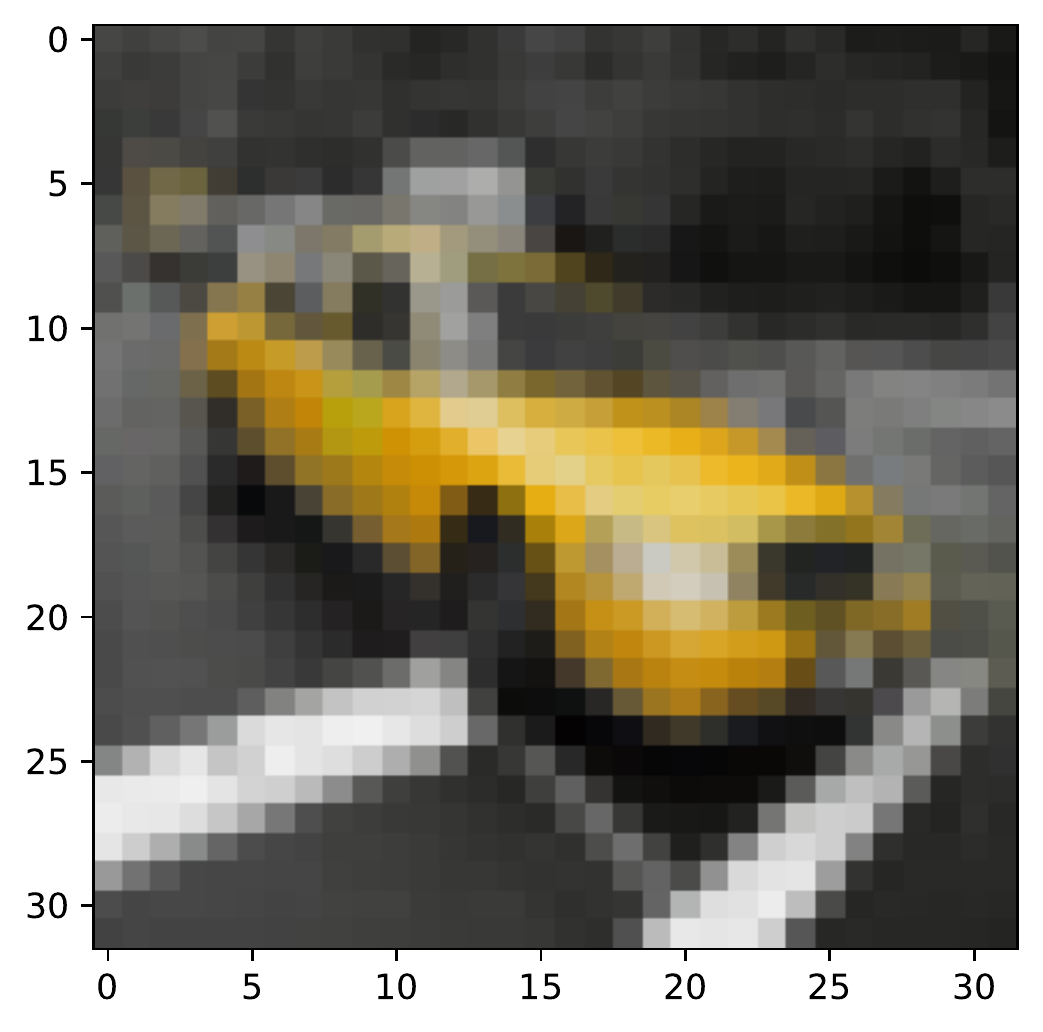} \\
%%%%%%%%
%%%%%%%%%%%%%%%%%%%%%%%%%%%%%%%%%%%%%%%%%%%%%%%%
{\footnotesize i) Identity QRPNN} &  
\parbox{0.22\columnwidth}{\centering \footnotesize j) High-order QRPNN ($q=70$)} &
    \parbox{0.22\columnwidth}{\centering \footnotesize k) Potential-function QRPNN ($L=5$)} &  \parbox{0.22\columnwidth}{\centering \footnotesize l) Exponential QRPNN ($\alpha=40$)}  \\
\includegraphics[width=0.22\columnwidth]{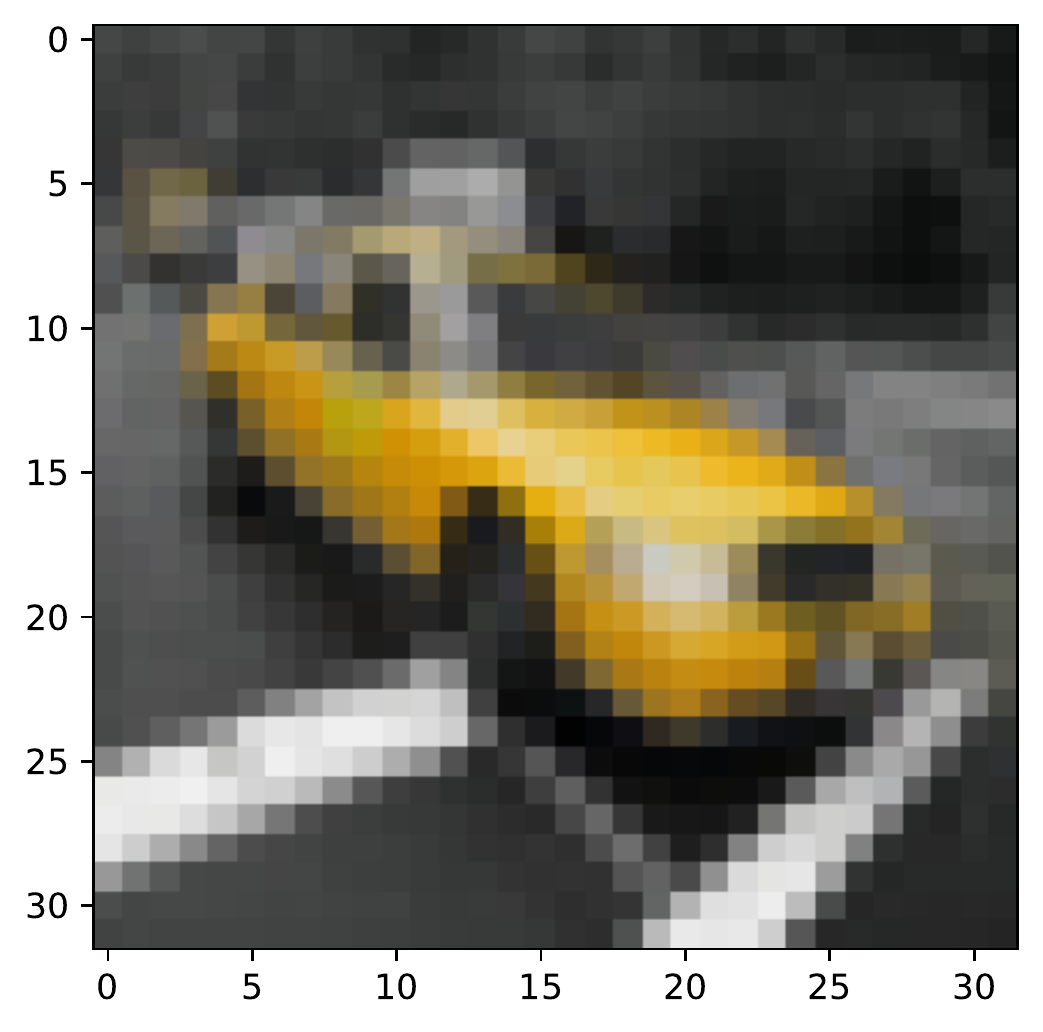} &
\includegraphics[width=0.22\columnwidth]{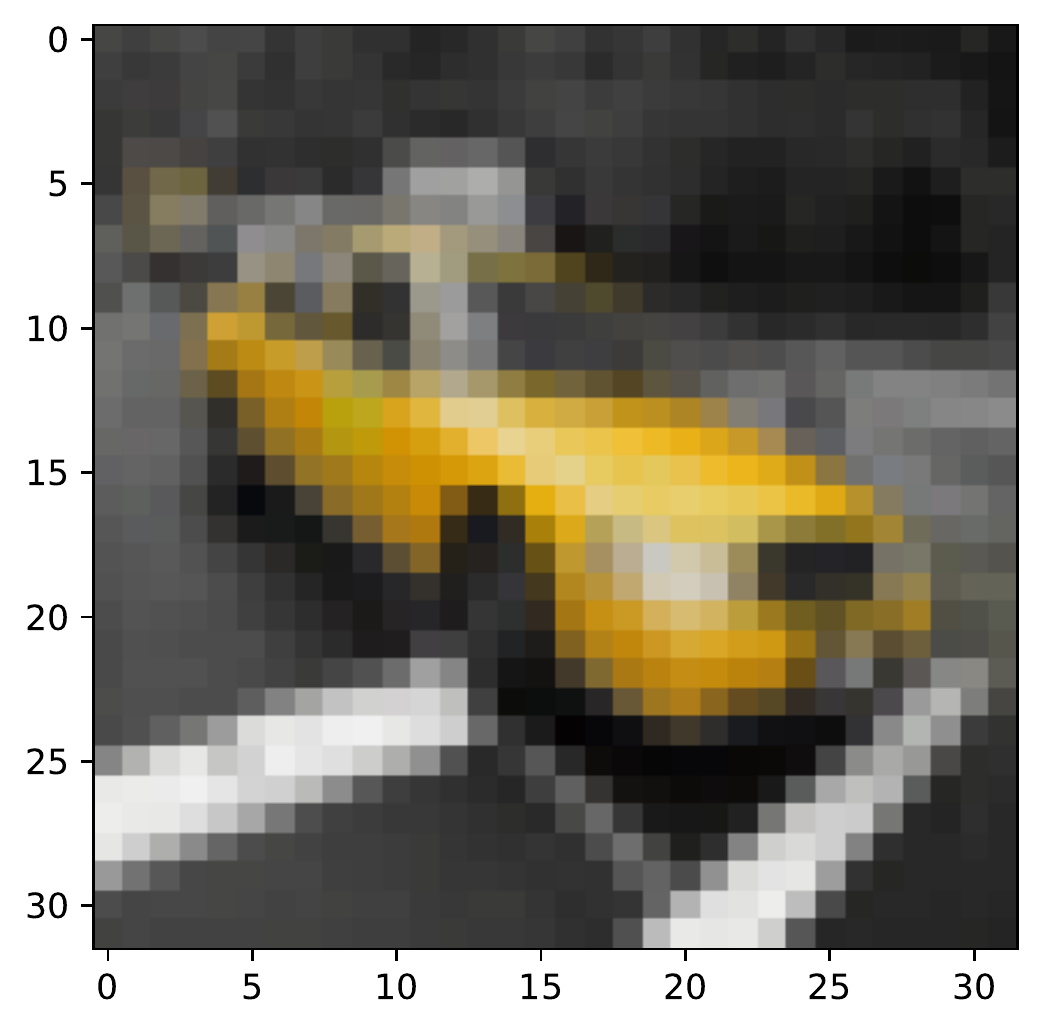} &
\includegraphics[width=0.22\columnwidth]{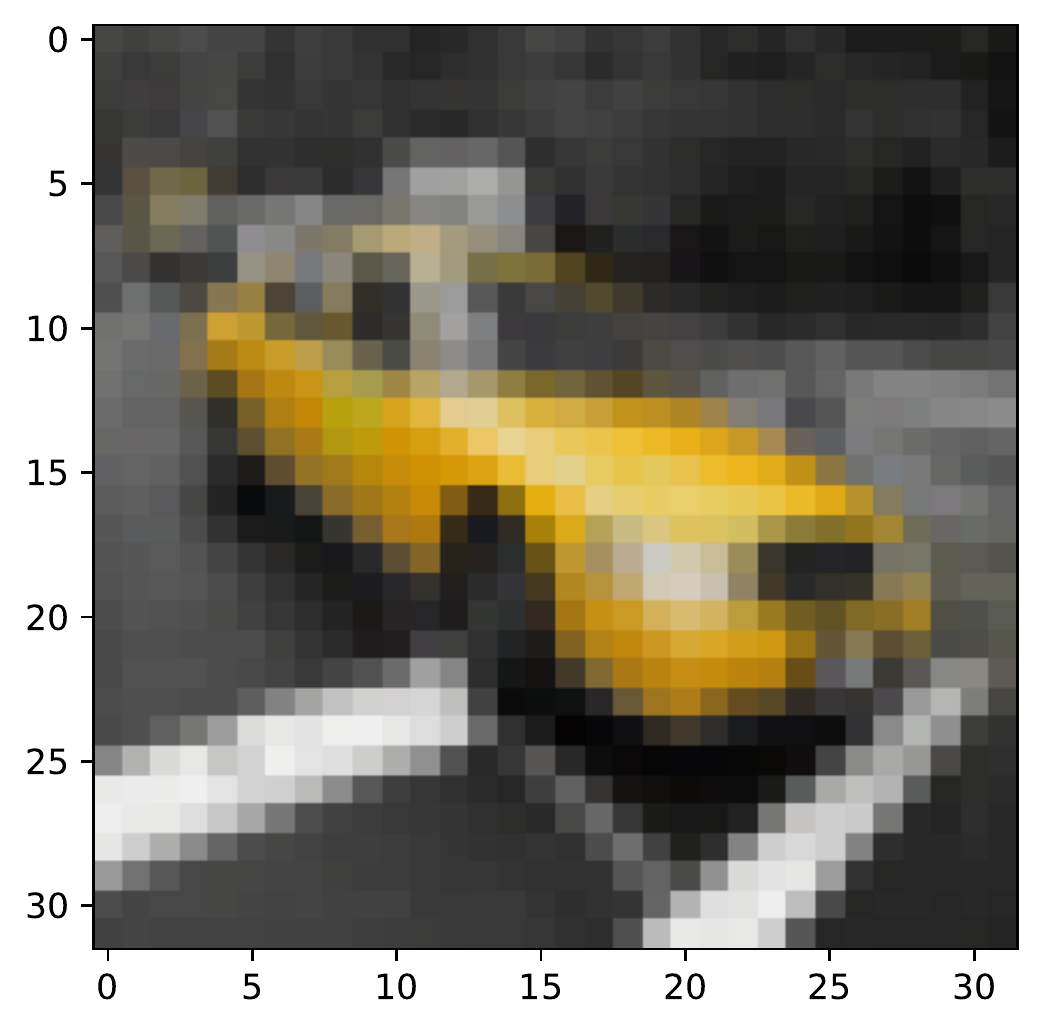} &
\includegraphics[width=0.22\columnwidth]{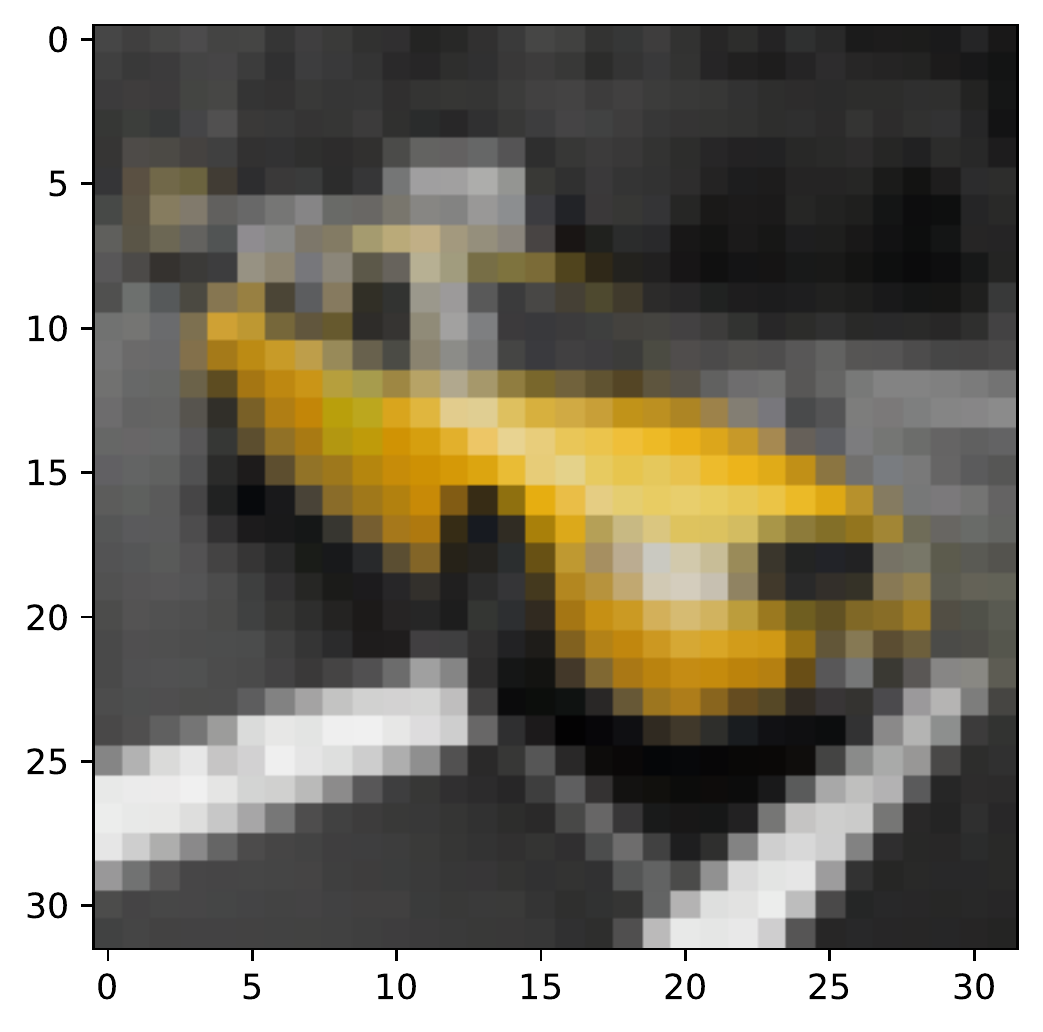}
    \end{tabular}
    \caption{Original color image, input image corrupted by Gaussian noise with standard deviation 0.1, and the corresponding images retrieved by the quaternion-valued associative memories.}
    \label{fig:images}
\end{figure}

In order to compare the performance of the quaternion-valued associative memories, we used color images from the CIFAR dataset \cite{cifar}. Recall that the CIFAR dataset contains 60000 RGB color images of size $32\times 32$. In this experiment, we randomly selected $p=200$ color images and converted them to unit quaternion-valued vectors $\vetu^1,\ldots,\vetu^p$. Similarly, we corrupted one of the $p$ selected images with Gaussian noise and converted it to a quaternion-valued vector $\vetx \in \mathbb{S}^n$. The corrupted vector $\vetx$ have been presented to associative memories designed for the storage of the fundamental memory set $\mathcal{U} = \{\vetu^1,\ldots,\vetu^p\} \in \mathbb{S}^n$. 

Figure \ref{fig:images} shows an original color image selected from the CIFAR dataset, a color image corrupted by Gaussian noise with standard deviation $0.1$, and the corresponding images retrieved by the associative memory models.
Note that the correlation-based QHNN as well as the identity QRCNN failed to retrieve the original image due to the cross-talk between the stored items. Although the projection-based QHNN yielded an image visually similar to the original cab's image, this memory model also failed to retrieve the original image due to the magenta pixels near the cab's bumpers. All the other associative memories succeed to retrieve the original image. Quantitatively, we say that an associative memory succeed to recall an stored image if the error given by the Euclidean norm $\|\vetu^1-\vety\|$, where $\vety$ denotes the retrieved quaternion-valued vector, is less than or equal to a tolerance $\tau = 10^{-4}$. Table \ref{tab:Errors} shows the error produced by the QHNN, QRCNN, and QRPNN memory models. This table also contains the error between the fundamental memory $\vetu^1$ and the quaternion-valued vector corresponding to the corrupted image. 
\begin{table}[t]
    \centering
    \begin{tabular}{||c|c||} \hline \hline
    Corrupted image:     & $20.8$  \\ \hline
    Correlation-based QHNN:     & $35.2$ \\
    Projection-based QHNN: & $1.9$ \\ \hline
    Identity QRCNN: & $41.2$ \\
    High-order QRCNN: & $8.3 \times 10^{-11}$\\
    Potential-function QRCNN: & $3.0 \times 10^{-15}$\\
    Exponential QRCNN: & $3.2 \times 10^{-10}$\\ \hline
    Identity QRPNN: & $9.0 \times 10^{-5}$ \\
    High-order QRPNN: & $3.3 \times 10^{-15}$\\
    Potential-function QRPNN: & $3.9 \times 10^{-15}$\\
    Exponential QRPNN: & $1.6 \times 10^{-15}$ \\ \hline \hline
    \end{tabular}
    \caption{Absolute error between the fundamental memory $\vetu^1$ and either the input or the quaternion-valued vector recalled by an associative memory model.}
    \label{tab:Errors}
\end{table}

\begin{figure}[t]
    \centering
    \includegraphics[width=1\columnwidth]{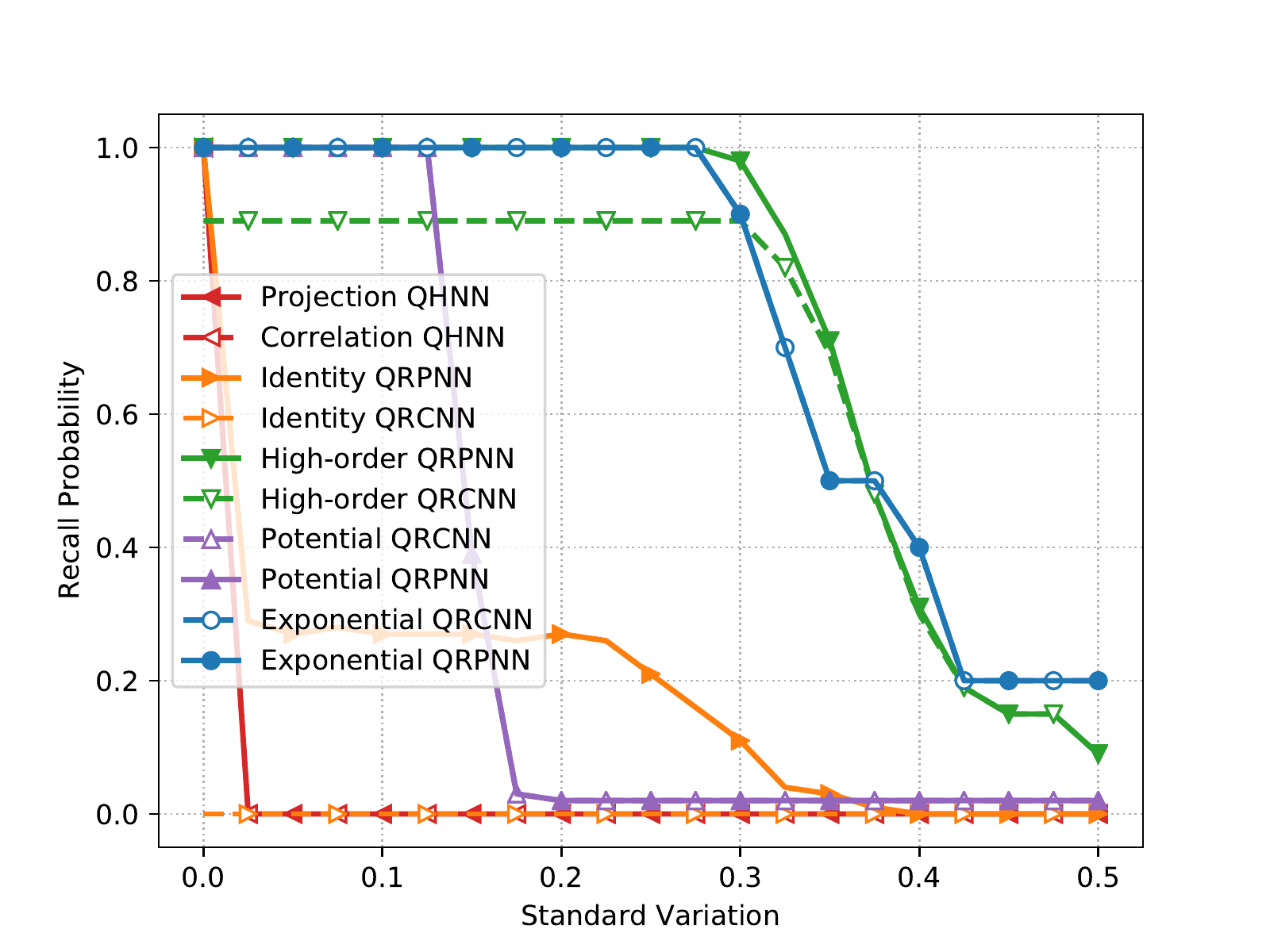}
    \caption{Recall probability of quaternion-valued associative memories by the standard deviation of the Gaussian noise introduced in the input.}
    \label{fig:ExpCIFAR}
\end{figure}
For a better comparison of the noise tolerance of the quaternion-valued associative memories, we repeated the preceding experiment 100 times. We also considered images corrputed by Gaussian noise with several different standard deviation values. Figure \ref{fig:ExpCIFAR} shows the probability of successful recall by the standard deviation of the Gaussian noise introduced in the input image. In aggreement with Theorem \ref{thm:fixed_points}, the QRPNN always succeeded to recall undistorted patterns (zero standard deviation). Note that the potential-function QRPNN and QRCNN coincided. From Theorem \ref{thm:RPNNxRCNN}, we conclude that these neural networks are in saturated mode. Furthermore, like the experiment described in the previous subsection,  the projection-based QHNN differs from the identity QRPNN. The latter, however, yielded larger recall probabilities because it circumvents the rotational invariance present in the QHNN model \cite{kobayashi16a}.

\begin{figure}[t]
    \centering
    \includegraphics[width=1\columnwidth]{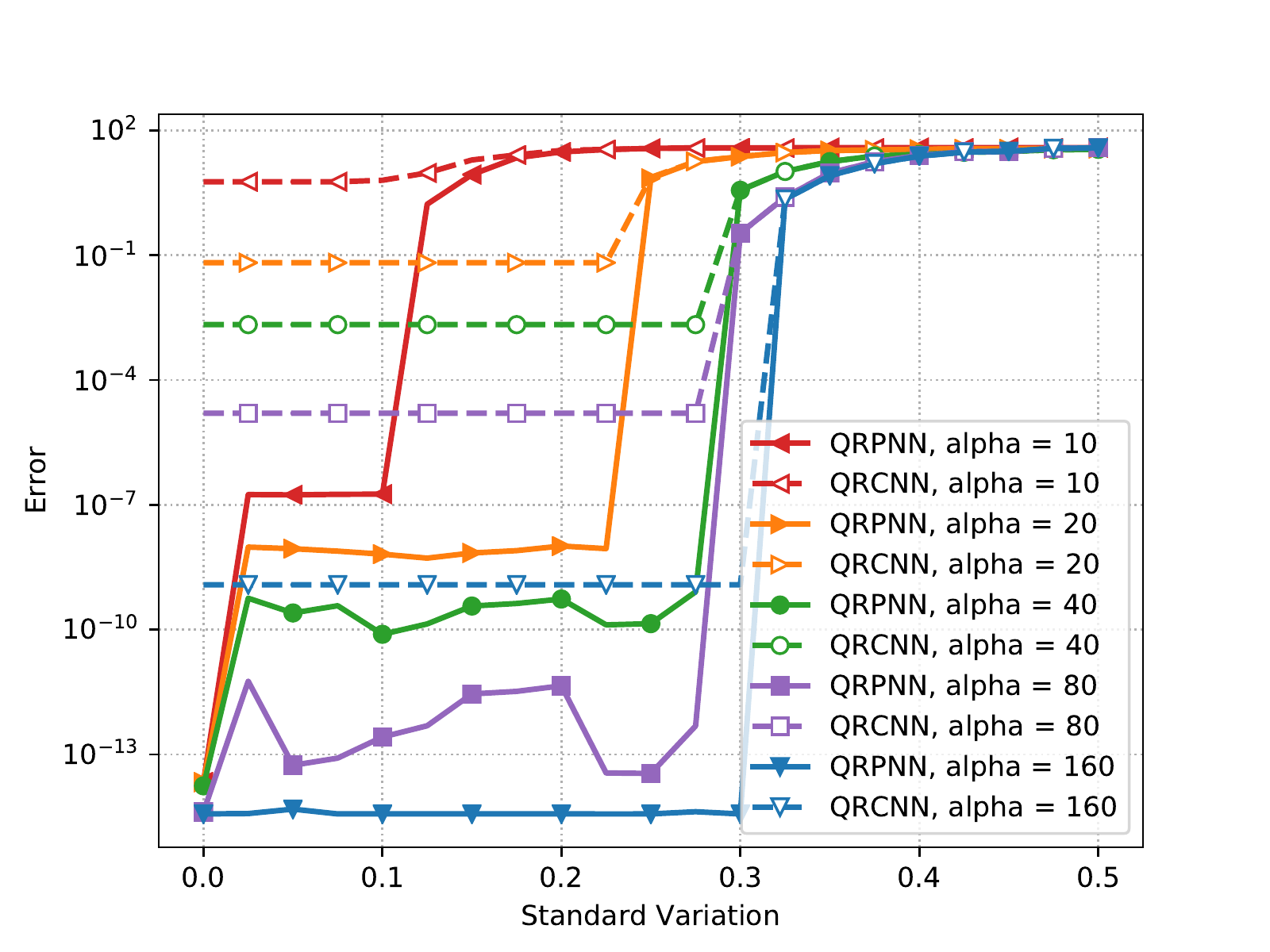}
    \caption{Error between the desired output and the retrieved quaternion-valued vector by the standard deviation of the Gaussian noise introduced in the input.}
    \label{fig:ExpCIFAR_Norm}
\end{figure}
Finally, we repeated the experiment used to generate Figure \ref{fig:ExpCIFAR} considering only the exponential QRPNN and QRCNN but with different values of the parameter $\alpha$. Moreover, to better discriminate the QRPNN and the QRCNN models, instead of computing the recall probability, we computed the Euclidean error between the desired output $\vetu^1$ and the retrieved vector $\vety$, that is, the error is given by $\|\vetu^1-\vety\|_2$. Figure \ref{fig:ExpCIFAR_Norm} depicts the average error by the standard deviation of the Gaussian noise introduced in the original color image. Note that the error produced by the exponential QRPNNs from an undistorted input are all around the machine precision, that is, around $10^{-14}$. Equivalently, the QRPNNs succeeded to recall undistorted images. Note also that the error produced by both QRPNN and QRCNN associative memories decreases as the parameter $\alpha$ increases. Nevertheless, the average error produced by the QRPNN are always below the corresponding QRCNN models. Finally, in accordance with Theorem \ref{thm:RPNNxRCNN}, the exponential QRPNN and QRCNN coincide when the parameter $\alpha$ is sufficiently large, i.e., $\alpha = 160$ in this experiment.

%%%%%%%%%%%%%%%%%%%%%%%%%%%%%%%%%%%%%%%%%%%%%%%%%%%%%%%%
\section{Concluding Remarks} \label{sec:concluding}

In this paper, we presented the quaternion-valued recurrent projection associative memories (QRPNNs). Briefly, QRPNNs are obtained by combining the projection rule with the quaternion-valued recurrent correlation associative memories (QRCNNs). In contrast to the QRCNNs, however, QRPNNs always exhibit optimal storage capacity (see Theorem \ref{thm:fixed_points}). Nevertheless, QRPNN and QRCNN coincide in the saturated mode (see Theorem \ref{thm:RPNNxRCNN}). Also, bipolar QRPNN and the recurrent kernel associative memory (RKAM) models coincide under mild conditions (see Theorem \ref{thm:RKAMxRPNN}). 
The computational experiments provided in Section \ref{sec:ComputationalExperiments} show that the storage capacity and noise tolerance of QRPNNs (including real-valued case) are greater than or equal to the storage capacity and noise tolerance of their corresponding QRCNNs. 

In the future, using recent results on hypercomplex-valued Hopfield neural network \cite{castro20nn}, we plan to extend the RCNN and RPNN models for other hypercomplex algebras such as hyperbolic numbers, commutative quaternions, and octonions. We also intent to investigate further the noise tolerance of the QRPNNs as well as to address the performance of the new associative memories for pattern reconstruction and classification.

\section*{Acknowledgments}

This work was supported in part by CNPq under grant no. 310118/2017-4, FAPESP under grant no. 2019/02278-2, and Coordena\c{c}\~ao  de Aperfei\c{c}oamento  de Pessoal de N\'ivel Superior - Brasil (CAPES) - Finance Code 001.

% \begin{exmp}

% \end{exmp}


%% References



%% New version of the num-names style
\bibliographystyle{elsarticle-num-names}
% \bibliography{references.bib}
\bibliography{main.bbl}






% \end{thebibliography}

\end{document}